\RequirePackage{fix-cm}
\documentclass[twocolumn]{svjour3}          
\smartqed  
\usepackage{graphicx}

\usepackage{amsmath}
\usepackage{algorithm}
\usepackage{algpseudocode}  
\usepackage{amsfonts}
\usepackage{enumerate}
\usepackage[accsupp]{axessibility}

\usepackage{natbib}

\usepackage{ragged2e}
\usepackage{appendix}
\usepackage{booktabs}
\usepackage{color}
\usepackage{multirow}
\usepackage{multicol}
\usepackage{bm}
\usepackage{bbm}
\usepackage{caption}
\usepackage[labelformat=simple]{subcaption}
\usepackage{hyperref}
\hypersetup{
    colorlinks=true,
    linkcolor=blue,
    filecolor=blue,      
    urlcolor=blue,
    citecolor=cyan,
}

\usepackage{marvosym}

\def\ie{{\em i.e.}}
\def\eg{{\em e.g.}}

\newcommand{\figref}[1]{Fig. \ref{#1}}
\newcommand{\tabref}[1]{Table \ref{#1}}

\newcommand{\secref}[1]{Section \ref{#1}}
\newcommand{\myPara}[1]{\vspace{.05in}\noindent\textbf{#1}}

\newcommand{\mc}[1]{\mathcal{#1}}

\newcommand{\mb}[1]{\mathbb{#1}}

\newcommand{\tabincell}[2]{\begin{tabular}{@{}#1@{}}#2\end{tabular}}

\begin{document}

\title{Temporal Action Localization with Enhanced Instant Discriminability
}


\author{Dingfeng Shi\textsuperscript{1}, 
        Qiong Cao\textsuperscript{2},
        Yujie Zhong\textsuperscript{3},
        Shan An\textsuperscript{2},
        Jian Cheng\textsuperscript{1},
        Haogang Zhu\textsuperscript{1,4},
        Dacheng Tao\textsuperscript{5}
}

\authorrunning{Dingfeng Shi, et al.} 

\institute{$\ast$~This~work~was~conducted~during~an~internship~at~JD.com.\\
\Letter~Corresponding~authors.\\
\Letter~Qiong~Cao~(mathqiong2012@gmail.com),\\
\Letter~Haogang Zhu~(haogangzhu@buaa.edu.cn),\\
\Letter~Jian~Cheng~(jian\_cheng@buaa.edu.cn).\\
\textsuperscript{1} Beihang University. \\
\textsuperscript{2} JD.com. \\
\textsuperscript{3} Meituran Inc. \\
\textsuperscript{4} Zhongguancun Laboratory. \\
\textsuperscript{5} University of Sydney. \\
}

\date{Received: date / Accepted: date}

\maketitle

\begin{abstract}
Temporal action detection (TAD) aims to detect all action boundaries and their corresponding categories in an untrimmed video. The unclear boundaries of actions in videos often result in imprecise predictions of action boundaries by existing methods. To resolve this issue, we propose a one-stage framework named TriDet. First, we propose a Trident-head to model the action boundary via an estimated relative probability distribution around the boundary. Then, we analyze the rank-loss problem (\ie~instant discriminability deterioration) in transformer-based methods and propose an efficient scalable-granularity perception (SGP) layer to mitigate this issue. 
To further push the limit of instant discriminability in the video backbone, we leverage the strong representation capability of pretrained large models and investigate their performance on TAD. Last, considering the adequate spatial-temporal context for classification, we design a decoupled feature pyramid network with separate feature pyramids to incorporate rich spatial context from the large model for localization.
Experimental results demonstrate the robustness of TriDet and its state-of-the-art performance on multiple TAD datasets, including hierarchical (multilabel) TAD datasets.
\end{abstract}

\section{Introduction}
\label{sec:introduction}
Temporal action detection (TAD) is a critical task in video understanding that consists of two subtasks: action recognition and action localization. In the deep learning era, the two mainstream methods, CNN-based methods and transformer-based methods, have made impressive progress in TAD. 
However, several unsolved problems in TAD make it a challenging task.

In object detection, the majority of objects have clear boundaries, such as outlines, which make them relatively easy to detect. However, the lack of clear boundaries is a significant issue in TAD. For instance, it can be challenging to pinpoint the exact frame that marks the boundary at the end of a long jump~\citep{alwassel2018diagnosing}. This issue makes accurate localization in TAD challenging. 

Existing methods address the problem from two main perspectives. Some studies~\citep{lin2018bsn,lin2019bmn,zeng2019graph,zhao2020bottom,long2019gaussian,nag2022multi,fu2023semantic} aim to determine the boundaries of action by relying on global features, possibly missing detailed information at each instant. Meanwhile, other studies directly predict boundaries based on a single local feature~\citep{zhang2022actionformer,paul2018w}, potentially with some other features~\citep{lin2021learning,qing2021temporal,zhao2021video}, but they do not consider the relation between adjacent instants around the boundary. 

To enhance localization learning, we posit that the relative response intensity of temporal features in a video can mitigate the impact of video feature complexity and increase localization accuracy. Based on this idea, we propose a one-stage action detector with a novel detection head called the Trident-head that is designed for action boundary localization. Instead of directly predicting the boundary offsets based on a single local feature, the proposed Trident-head models the action boundary via an estimated relative probability distribution of the boundary. The boundary offset is then computed based on the expected values of the neighboring bin set.

Furthermore, the transformer-based feature pyramid is utilized in several recent TAD methods~\citep{zhang2022actionformer,cheng2022tallformer,weng2022efficient}, demonstrating encouraging outcomes. However, the video backbone's video features often exhibit significant similarities among snippets, which are further amplified by self-attention, leading to the rank-loss problem (\ie~discriminability deterioration)~\citep{shi2023tridet}.  
Fortunately, the success of the previous transformer-based layers in TAD relies primarily on their macro-architecture, namely, how the normalization layer and feed-forward network (FFN) are connected, rather than the self-attention mechanism.
We, therefore, propose an efficient convolutional-based layer, termed the scalable-granularity perception (SGP) layer, to alleviate the two 
abovementioned problems of self-attention. SGP comprises two primary branches, which serve to increase the discrimination of features in each instant and capture temporal information with different scales of receptive fields.

Additionally, most existing TAD methods utilize an action classification network (\eg~SlowFast\citep{feichtenhofer2019slowfast}, I3D~\citep{carreira2017quo} and TSN~\citep{wang2018temporal}) that is pretrained on a single dataset as a backbone. Therefore, the features extracted from those backbones are often not sufficiently distinct for boundary localization. To further push the limit of discriminability in the video backbone, we leverage the large image and video models and improve the localization accuracy for the TAD task. 
Namely, we utilize two types of pretrained large models: \emph{temporal-level} (the simplified term for spatial-temporal-level) and \emph{spatial-level} backbones, respectively. 

The temporal-level backbone (\eg~VideoMAEv2\\\citep{wang2023videomaev2}) efficiently extracts a comprehensive representation within a specific temporal window, resulting in excellent detection performance. However, this representation is not precisely aligned with frame information, leading to potentially inaccurate localization. 
On the other hand, the visual context provided by the spatial-level backbone (e.g., DINOv2~\citep{oquab2023dinov2}, MoCov2~\citep{chen2020mocov2}), such as the appearance of specific objects (e.g., cigarettes in the smoking action), often determines the start and end of an action. Motivated by this, we propose a decoupled feature pyramid network (FPN) to fuse information from two backbone networks: VideoMAEv2 and DINOv2, which shows better results compared to other straightforward fusion methods. Concretely, we construct two separate feature pyramids based on the two backbones, namely, the temporal-level feature pyramid and the spatial-level feature pyramid. Then, the temporal-level feature pyramid is directly fed into the classification head, while both the temporal-level feature pyramid and spatial-level feature pyramid are combined through element-by-element summation along each level of the pyramid and fed into the localization head (\ie~Trident-head).
By fully leveraging the benefits of both types of backbone networks, the decoupled FPN aids in enhancing localization.

Experimental results on several conventional TAD and multilabel TAD datasets show that TriDet is a state-of-the-art action detector, and extensive ablation experiments demonstrate its robustness.

The CVPR 2023 version~\citep{shi2023tridet} investigates the rank-loss issue of the transformer in TAD and introduces the SGP layer and Trident-head for more precise localization. The extensions of this work include the following: (1) we conduct experiments and analyze the application of temporal-level and spatial-level visually pretrained large models in TAD; (2) we propose a straightforward yet comprehensive feature pyramid network (FPN) that incorporates spatial-level context; (3) we build a comprehensive model that can adapt to multilabel TAD tasks and provide further results on two multilabel detection datasets MultiTHUMOS and Charades; (4) we conduct detailed experiments to analyze different variants of TriDet; (5) The code will be released to 
\url{https://github.com/dingfengshi/tridetplus}

\section{Related Work}
\subsection{Temporal Action Detection}
Temporal action detection (TAD) is a critical task for video understanding that aims to detect all action segments along with their boundary location and classification from untrimmed video~\citep{lee2023decomposed,lee2020background}. In the deep learning era, existing temporal action detection methods can be divided into two categories: two-stage methods and one-stage methods. 

Two-stage methods~\citep{xu2020g,zeng2019graph,zhu2021enriching,sridhar2021class,qing2021temporal,nag2023difftad} comprise two independent networks: a proposal generation network and a classification network. Most previous works~\citep{lin2018bsn,lin2019bmn,lin2020fast,chen2022dcan,escorcia2016daps,liu2021multi} focus on the proposal generation phase. Specifically, some works~\citep{lin2019bmn,lin2018bsn,chen2022dcan} predict the probability of the action boundary and match the start and end instants densely based on the prediction score. Anchor-based methods~\citep{lin2020fast,escorcia2016daps} classify actions from specific anchor windows. However, two-stage methods are limited by high complexity and cannot be trained end-to-end. 

One-stage methods perform localization and classification with a single network. Some previous works~\citep{yang2020revisiting,lin2021learning,yang2022basictad,bhosale2023diffsed} build this hierarchical architecture based on a convolutional network (CNN). However, a performance gap remains between CNN-based and the latest TAD methods. 

\subsection{Object detection}
Object detection is a twin task of TAD. General focal loss~\citep{li2020generalized} transforms bounding box regression from learning a Dirac delta distribution to a general distribution function. They utilize multiple bins to predict boundaries. However, these bins do not directly correspond to actual image information and necessitate an additional loss function to aid in convergence. Some methods~\citep{howard2017mobilenets,chollet2017xception,liu2022convnet} use depth-wise convolution to model the network structure, and some branched designs~\citep{szegedy2017inception,hu2018squeeze} show high generalization ability. These approaches are enlightening for the architecture design of TAD.

\subsection{Transformer Based Method}
Inspired by the great success of the transformer in the field of machine translation and object detection, some recent  works~\citep{zhang2022actionformer,shi2022react,tan2021relaxed,cheng2022tallformer,liu2022end,liu2022empirical} adopt the attention mechanism in TAD to improve the detection performance. For example, some works\citep{tan2021relaxed,shi2022react,liu2022end} detect action with the DETR-like transformer-based decoder~\citep{carion2020end}, which models action segments as a set of learnable segments. Other works~\citep{zhang2022actionformer,cheng2022tallformer} extract a video representation with a transformer-based encoder. 

However, most of these methods are based on \emph{local} behavior.
Namely, they conduct attention operations only in a local window, which introduces inductive bias similar to that of CNN but with larger computational complexity and additional limitations.

In addition, Dong et al.~\citep{dong2021attention} analyze the pure self-attention operation that causes the token features to converge to a rank-1 matrix at a double exponential rate during initialization. This inspired us to analyze the issue in TAD from the perspective of feature rank. However, their analysis primarily focuses on the initialization phase and lacks an examination of specific applications and potential enhancements. We dive deeper into the rank-loss problem that emerges during the training of the TAD detector and propose an effective improvement scheme, the SGP layer, utilizing convolution.

\subsection{Large-scale Pretrained Model}
Recently, large-scale self-supervised pretrained models such as BERT~\citep{devlin2018bert} and GPT-4~\citep{openai2023gpt4} have shown impressive results in the field of neural language processing (NLP). The large models trained using massive amounts of data have significantly improved various downstream tasks.

Motivated by the success seen in NLP, some researchers have begun to build visual foundation models by pretraining large-scale visual models~\citep{bao2021beit,kirillov2023segment,chen2020simple,xie2022simmim,chen2020mocov2}, including image foundation models and video foundation models~\citep{feichtenhofer2022masked,lin2022frozen,wang2022bevt}.
For the video foundation model, VideoMAEv2~\citep{wang2023videomaev2} constructs a large model with billions of parameters to learn the temporal-visual representation of videos from multiple video datasets in a self-supervised manner. 
For the image foundation model, DINOv2~\citep{oquab2023dinov2} builds a large-scale curated image dataset for large ViT model training and distilling. 

Although these large models have demonstrated strong performance in certain downstream tasks, their utilization and specific impact on TAD have not been thoroughly investigated. In this study, we aim to examine the impact of two types of backbones in TAD and explore methods to fully harness the potential of both backbone networks.

\begin{figure*}[t]
    \centering{
    \includegraphics[width=0.95\linewidth]{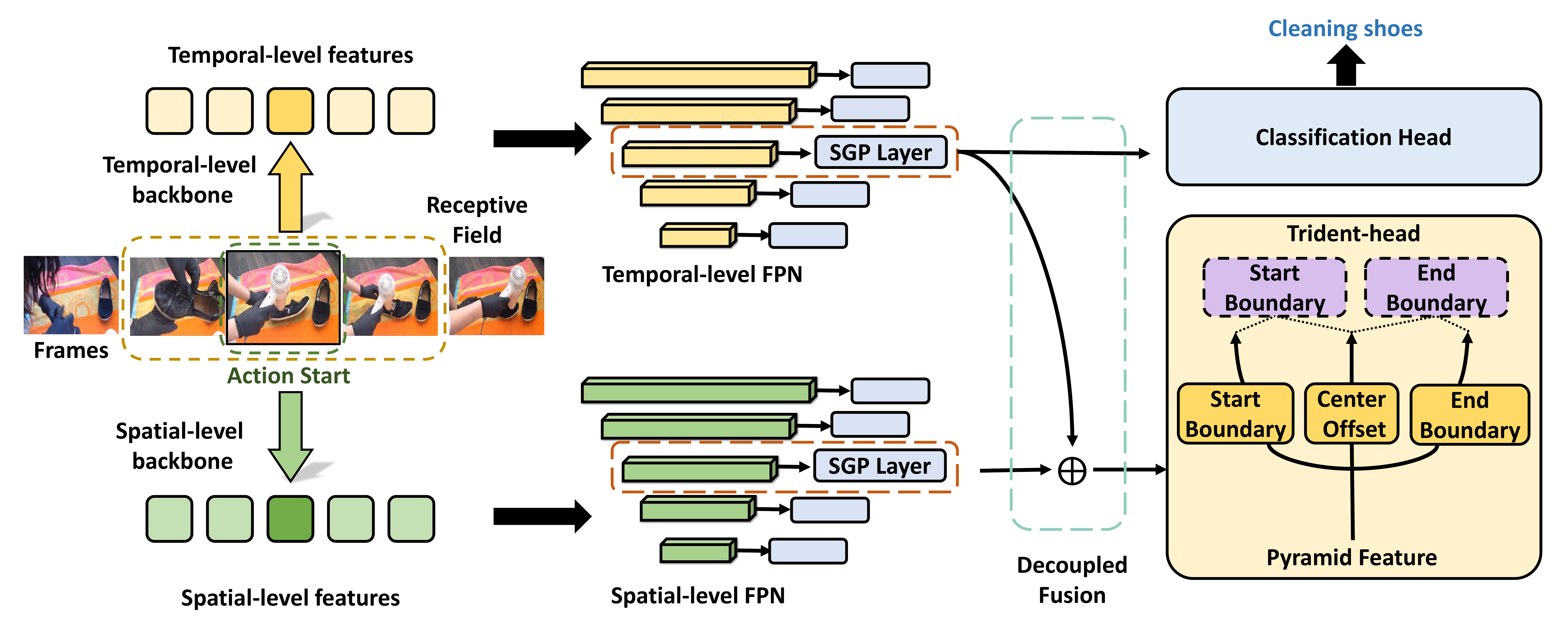}
    
  \caption{Illustration of TriDet. We adopt temporal-level and spatial-level backbones to construct pyramid features with the scalable-granularity perception (SGP) layer. The temporal-level features at each level are fed into a shared-weight classification head to determine the action categories. Additionally, the Trident-head takes in a fusion of temporal-level and spatial-level features to obtain the offset for each instant, and estimates the boundary offset based on the relative distribution predicted by three branches: start boundary, end boundary and center offset.}
  \label{fig:framework} 
  }
\end{figure*}

\section{Method}
\subsection{Preliminaries} 
\myPara{Problem definition.} We first give a formal definition of the TAD task. Specifically, given a set of untrimmed videos $\mc{D}=\{\mc{V}_i\}_{i=1}^{n}$, we have a set of temporal visual features $X_i =\{x_t\}_{t=1}^{T_i}$ from each video $\mc{V}_i$, where $T_i$ corresponds to the number of instants, and $K_i$ segment labels $Y_i=\{s_k,e_k,c_k\}_{k=1}^{K_i}$ with the action segment start instant $s_k$, the end instant $e_k$ and the corresponding action category $c_k$. TAD aims to detect all segments $Y_i$ based on the input feature $X_i$. 

\subsection{Method Overview}
Our goal is to build a simple and efficient one-stage temporal action detector. 
As shown in \figref{fig:framework}, the overall architecture of our detector consists of two branch feature backbones (temporal-level and/or spatial-level backbones). These backbones are accompanied by two corresponding feature pyramids, which undergo decoupled fusion and are fed into a classification head as well as a boundary-oriented Trident-head.

First, the video features are extracted using a pretrained temporal-level backbone (\eg~ VideoMAEv2\\\citep{wang2023videomaev2} or SlowFast~\citep{feichtenhofer2019slowfast}) or pretrained spatial-level backbone (\eg{~DINOv2~\citep{oquab2023dinov2}}). Next, the SGP feature pyramid for each backbone is built to tackle actions with various temporal lengths, similar to some recent TAD works~\citep{lin2021learning,zhang2022actionformer,cheng2022tallformer}.
Namely, the two backbone features are iteratively downsampled, and each scale level is processed with a proposed SGP layer to enhance the interaction between features with different temporal receptive fields,  resulting in the temporal-level feature pyramid and spatial-level feature pyramid.
Then, the temporal-level feature pyramid is fed directly into
the classification head, while both the temporal-level feature pyramid
and spatial-level feature pyramid are combined through element-by-
element summation along each level of the pyramid and fed
into the Trident-head.

Last, action segments are detected by a designed boundary-oriented Trident-head. 
We elaborate on the proposed modules in the following. 

\subsection{Feature Pyramid with SGP Layer}
\label{sec:sgp}
\myPara{The feature pyramid network (FPN).} The feature pyramid is obtained by first downsampling the output features of the video backbone network several times via max-pooling (with a stride of 2). The features at each pyramid level are then processed using transformer-like layers (e.g. ActionFormer~\citep{zhang2022actionformer}). 

\begin{figure}[t]
    \centering{
    \includegraphics[width=0.87\linewidth]{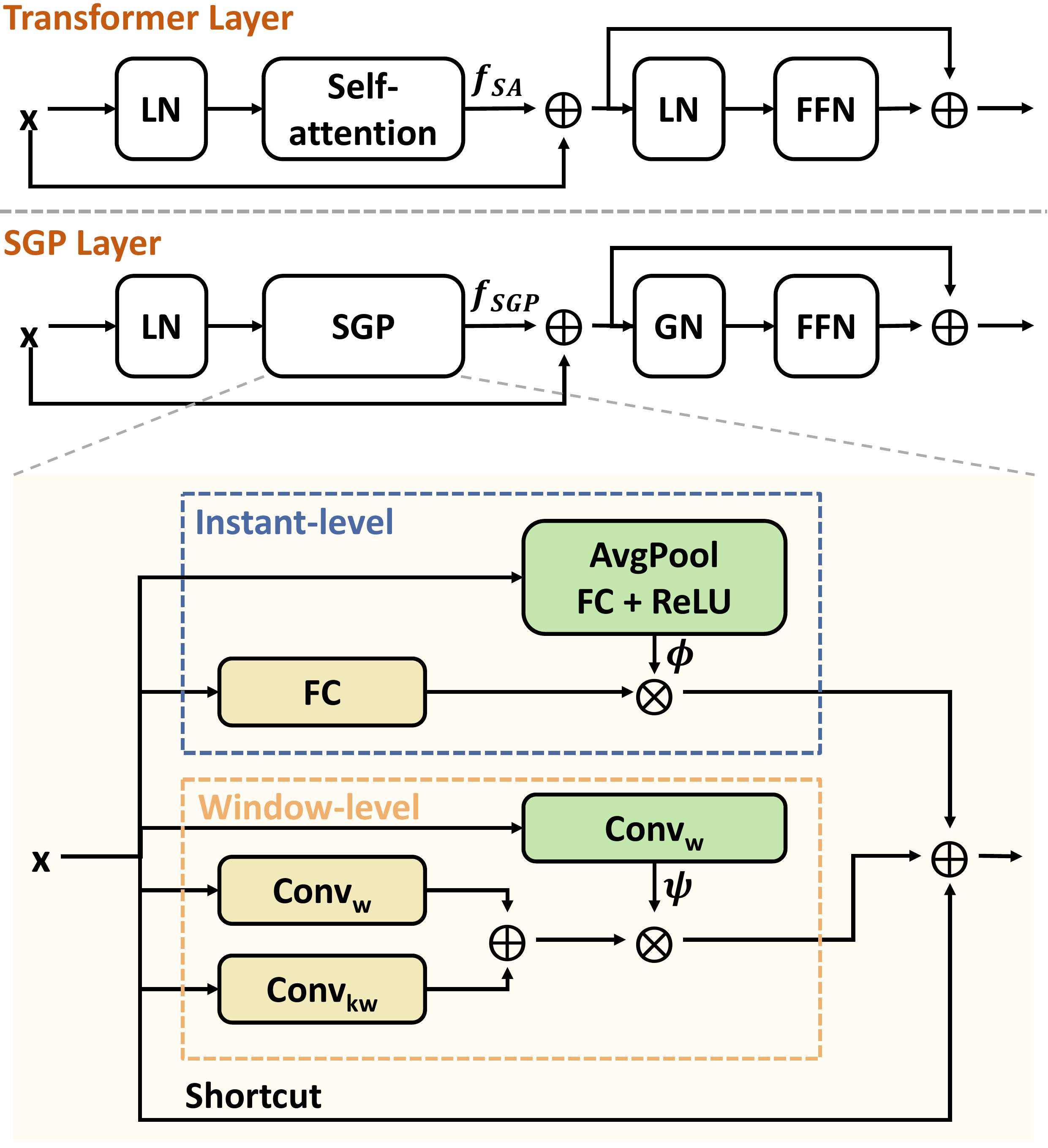}
    
  \caption{Illustration of the structure of the SGP layer. We replace the self-attention and the second layer normalization (LN) with SGP and group normalization (GN), respectively. }
  \label{fig:sgp}
  }
\end{figure}

Current transformer-based methods for TAD tasks rely primarily on the macro-architecture of the transformer (See \secref{sec:transformer} for details), rather than the self-attention mechanisms. Specifically, SA encounters two main issues: the rank-loss problem across the temporal dimension and high computational overhead.

\myPara{Limitation 1: the rank-loss problem.}
In~\citep{dong2021attention}, the authors examine the pure self-attention operation, which leads to the convergence of the token feature matrix to a rank-1 matrix at a double exponential rate during the initialization phase. This means that the feature sequences become more similar and less distinguishable with depth, which is referred to as the rank-loss problem.
We have also observed the occurrence of rank loss during the training phase in the TAD task (see \secref{sec:rank_loss}). The temporal feature sequences extracted from the pretrained backbone network already show a high level of similarity, which makes it challenging to detect boundaries. However, the self-attention approach exacerbates this similarity issue and is detrimental to accurate localization.

We argue that the rank-loss problem in TAD arises because the probability matrix in self-attention (\ie~soft-max($QK^T$)) is \emph{nonnegative} and \emph{the sum of each row is 1}, indicating the outputs of SA are a \emph{convex combination} of the value feature $V$. Considering that pure layer normalization~\citep{ba2016layer} projects features onto the hypersphere in high-dimensional space, we analyze the degree of their distinguishability by studying the maximum angle between features within the instant features. We demonstrate that the maximum angle between features after the \emph{convex combination} is less than or equal to that of the input features, resulting in increased similarity between features (as outlined in the supplementary material). 
A straightforward approach to mitigate this problem could be the replacement of the self-attention mechanism with a convolutional layer, given that it does not impose any constraints on the weights of its convolutional kernel. However, employing this approach leads to a performance drop (see SA-to-CNN in \secref{sec:transformer}), which can be attributed to the limited fitting power of a single convolutional layer. Thus, it still necessitates innovative designs to enhance the performance of convolutional structures.


\myPara{Limitation 2: high computational complexity.}
In addition, the dense pair-wise calculation (between instant features) in self-attention brings a high computational overhead and therefore decreases the inference speed.

Based on the above discovery, we propose an SGP layer to effectively capture multigranularity action information while suppressing the issues of rank loss and high computation complexity. The major difference between the transformer layer and the SGP layer is the replacement of the self-attention module with the fully convolutional module SGP. The successive layer normalization\citep{ba2016layer} (LN) is also changed to group normalization\citep{wu2018group} (GN).

As shown in \figref{fig:sgp}, SGP contains two main branches: an instant-level branch and a window-level branch. 
In the instant-level branch, we interact the features of each instant with the global average features in order to capture the overall video context. The window-level branch is designed to introduce the semantic content from a wider receptive field with a branch $\psi$ to help dynamically focus on the features of each scale.
Mathematically, given the temporal feature $X \in \mc{R}^{T \times D}$, the SGP can be written as:
\begin{equation}
    f_{SGP} = \phi(X)FC(X) + \psi(X)(Conv_w(X) + Conv_{kw}(X))+ X,
\end{equation}
where $FC$ and $Conv_w$ denote a fully-connected layer for each instant and a 1-D depth-wise convolution layer~\citep{chollet2017xception} over the temporal dimension with window size $w$, respectively.
As a signature design of SGP, $k$ is a scalable factor that captures a larger granularity of temporal information.
The video-level average feature $\phi(X)$ and branch $\psi(X)$ are given as
\begin{align}
    \phi(X) &= ReLU(FC(AvgPool(X))),\\
    \psi(X) &= Conv_w(X),
\end{align}
where $AvgPool(X) \in \mc{R}^{1 \times D}$ denotes the average pooling operation applied to all features along the temporal dimension. Here, both $\phi(X)$  (which will be replicated to $\mc{R}^{T \times D}$) and $\psi(X)\in \mc{R}^{T \times D}$ perform element-wise multiplication with the mainstream feature.

The resultant SGP-based feature pyramid can achieve better performance than the transformer-based feature pyramid while being much more efficient.

\begin{figure}[t]
    \centering{
    \includegraphics[width=0.90\linewidth]{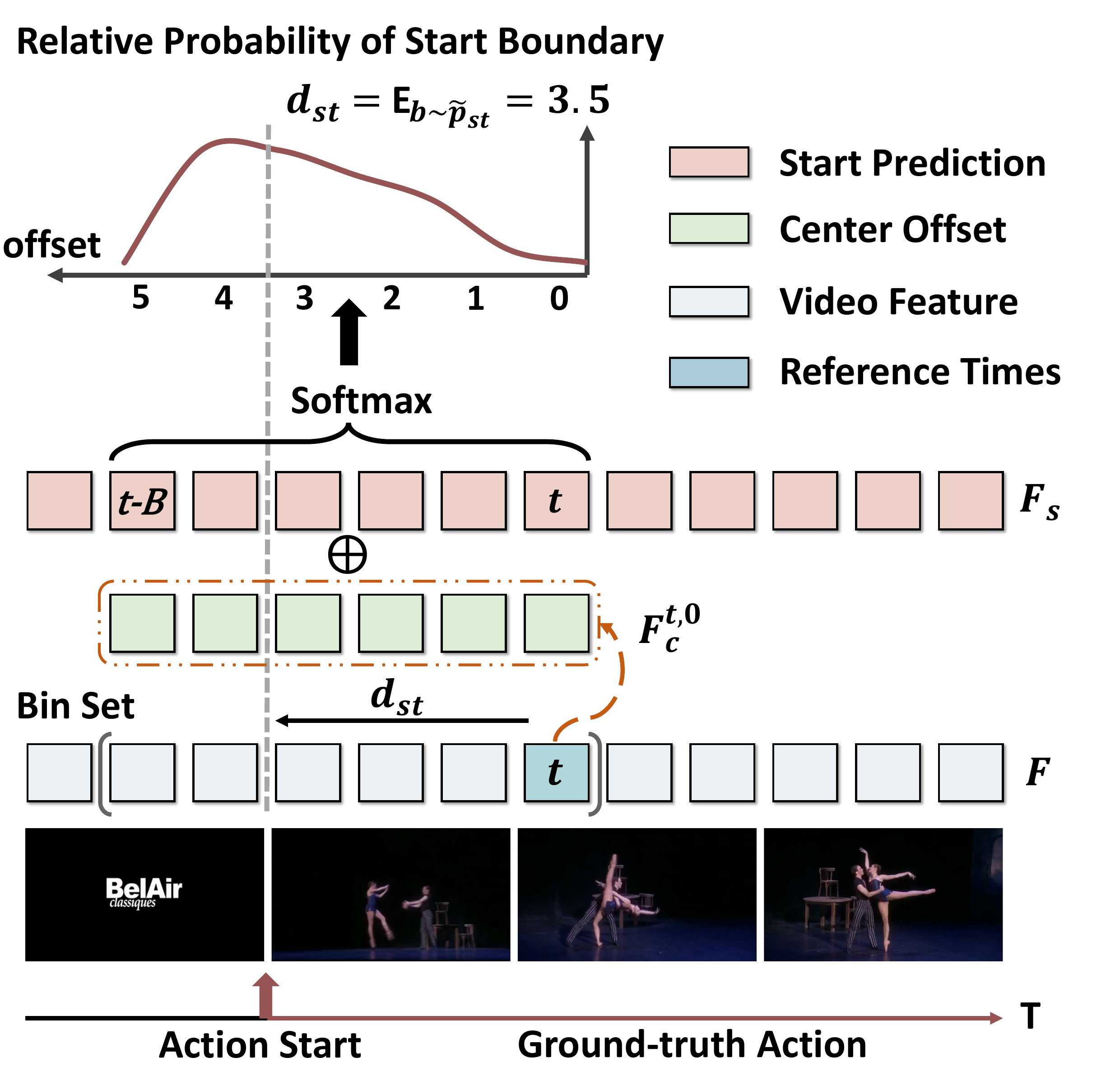}
    
  \caption{The boundary localization mechanism of Trident-head. We predict the boundary response and the center offset for each instant. At instant t, the predicted boundary response in the neighboring bin set is summed element-wise with the center offset corresponding to instant t, which is further estimated as the relative boundary distribution.
  Finally, the offset is computed based on the expected value of the bin.
  }
  \label{fig:head} 
  }
\end{figure}

\subsection{Trident-head with Relative Boundary Modeling}
\label{sec:head}

\myPara{Intrinsic property of action boundaries.}
Regarding the detection head, some existing methods directly regress the temporal length~\citep{zhang2022actionformer} of the action at each instant of the feature and refine with the boundary feature\citep{lin2021learning,qing2021temporal}, or ~\citep{lin2018bsn,lin2019bmn,zeng2019graph} simply predict an \emph{actionness} score (indicating the probability of being an action). These simple strategies suffer from a problem in practice: imprecise boundary predictions, due to the intrinsic property of actions in videos. Namely, the boundaries of actions are typically not obvious, unlike the boundaries of objects in object detection. Intuitively, a statistical boundary localization method can reduce uncertainty and facilitate more precise boundaries.

\myPara{Trident-head.}
In this work, we propose a boundary-oriented Trident-head to precisely locate action boundaries based on  relative boundary modeling, \ie~considering the relation of features in a certain neighboring bin set and obtaining the relative probability of being a boundary for each instant in that set. The Trident-head consists of three components, namely, a start head, an end head, and a center-offset head, which are designed to locate the start boundary and end boundary and to capture the localization context from the center instant of the action, respectively. The Trident-head can be trained end-to-end with the detector.

Concretely, as shown in \figref{fig:head}, given a sequence of features $F\in\mc{R}^{T \times D}$ output from the feature pyramid, we first obtain three feature sequences with the convolutional layer from the three branches, namely, start prediction $F_{s} \in \mc{R}^{T}$, end prediction $F_{e} \in \mc{R}^{T}$ and center offset $F_{c} \in \mc{R}^{T \times 2 \times (B+1)}$, where $B$ is the number of bins for boundary prediction, $F_{s}$ and $F_{e}$ characterize the response value for each instant as the starting or ending point of the action, respectively. The center offset $F_{c}$ aims to estimate two conditional distributions $P(b_{st}|t)$ and $P(b_{et}|t)$, which represent the probability that each instant $b_{st}/b_{et}$ (in start/end bin set) serves as a boundary when the instant $t$ is the midpoint of an action. Then, we model the boundary distance by combining the outputs of the boundary head and center-offset head:
\begin{align}
    \widetilde{P}_{st} &= Softmax(F_{s}^{[(t-B):t]} + F_{c}^{t,0}),\\
    d_{st} &= \mb{E}_{b\sim \widetilde{P}_{st}}[b] \approx \sum_{b=0}^B(b\widetilde{P}_{stb}),
\end{align}
where $F_{s}^{[(t-B):t]} \in \mc{R}^{B+1}$, $F_{c}^{t,0} \in \mc{R}^{B+1}$ are the response values in the start bin set of instant $t$ and the center offsets predicted by instant $t$ only, respectively, and $\widetilde{P}_{st}$ is the \emph{relative probability}, which represents the probability of each instant being the start of an action within the bin set. Then, the distance between instant $t$ and the start instant of action segment $d_{st}$ is given by the expectation of the adjacent bin set. Similarly, the offset distance of the end boundary $d_{et}$ can be obtained by
\begin{align}
    \widetilde{P}_{et} &= Softmax(F_{e}^{[t:(t+B)]} + F_{c}^{t,1}),\\
    d_{et} &= \mb{E}_{b\sim \widetilde{P}_{et}}[b] \approx \sum_{b=0}^B(b\widetilde{P}_{etb\textbf{}})
\end{align}

All heads are simply modeled in three layers of convolutional networks and share parameters at all feature pyramid levels to reduce the number of parameters.

\myPara{Combination with feature pyramid.} 
We apply the Trident-head in a predefined local bin set, which can be further improved by combining it with the feature pyramid.
In this setting, features at each level of the feature pyramid  share the same small number of bins $B$ (\eg~16), and  the corresponding prediction for each level $l$ can be scaled by $2^{l-1}$, which can significantly help to stabilize the training process. 

Formally, for an instant in the l-th feature level $t^l$,  SGP estimates the boundary distance $\hat{d}_{st}^l$ and $\hat{d}_{et}^l$ with the Trident-head described above; then, the segments $a=(\hat{s}_t, \hat{e}_t)$ can be decoded by 
\begin{align}
    \hat{s}_t &= (t-\hat{d}_{st}^l)\times 2^{l-1},\\
    \hat{e}_t &= (t+\hat{d}_{et}^l)\times 2^{l-1}.
\end{align}

\myPara{Extension to multilabel task.} It would be easy to adapt the Trident-head to a multilabel temporal action detection task by setting the outputs of the three branches to be category-dependent and assigning the positive sample positions and their corresponding regression value according to the category.

Concretely, the three branches in Trident-head are $F_{s} \in \mc{R}^{C\times T}$, $F_{e} \in \mc{R}^{C\times T}$ and $F_{c} \in \mc{R}^{C \times T \times 2 \times (B+1)}$. The final offset distance for each instant $d \in \mc{R}^{C\times T \times 2}$ is calculated in the same way as before. 

During the training phase, we employ central sampling to select classification and regression samples for each ground truth segment and allocate them to their respective instants and categories for regression. For classification, we adopt the multilabel binary objective and use focal loss~\citep{lin2017focal} to optimize the classification head. 

In the test, we select predicted segments based on the classification scores and use the corresponding predicted offset to decode the predicted segments. 

\myPara{Comparison with existing methods that have explicit boundary modeling.} Many previous methods improve boundary predictions. We divide them into two broad categories: prediction based on sampling instants in segments~\citep{lin2019bmn,liu2022end,shi2022react} and prediction based on a single instant. The first category predicts the boundary according to the global features of the predicted segments.
These methods consider only global information instead of detailed information at each instant. The second category directly predicts the distance between an instant and its corresponding boundary based on the spatial-level feature~\citep{lin2021learning,zhang2022actionformer,zhao2021video,qing2021temporal}. Some of these methods refine the segment with boundary features~\citep{lin2021learning,qing2021temporal,zhao2021video}.

However, they do not take the relation (\ie~relative probability of being a boundary) of adjacent instants into account.
The proposed Trident-head differs from these two categories and shows superior performance in precise boundary localization.

\subsection{Enhance TAD with Large-scale Pre-trained Models}
\subsubsection{Large-scale temporal-level backbone.}
\label{sec:videomae}
Traditional TAD tasks rely on temporal-level backbones to extract temporal feature sequences. However, the temporal features extracted from most of these backbones lack high distinguishability due to training from inadequate data. Fortunately, recent visual foundation models pretrained from massive amounts of data have shown impressive results in a variety of downstream tasks, inspiring us to utilize large models to enhance the performance in TAD.
Therefore, we utilize VideoMAEv2~\citep{wang2023videomaev2} as our temporal-level backbone. It is pretrained on the UnlabeledHybrid dataset and then fine-tuned on the K710 dataset.

\myPara{The advantages and disadvantages.} 
In \secref{sec:vmae_slowfast}, we analyze the advantages and disadvantages of VideoMAEv2 by comparing the detection results between VideoMAEv2 and SlowFast, a commonly used TAD backbone. Here, we present a simplified analysis.

VideoMAEv2 aims to learn a temporal-level representation from a short clip composed of multiple frames (\eg~ 16), and the representation benefits from a large model and a large amount of training data~\citep{wang2023videomaev2}. Therefore,  VideoMAEv2 has a greater advantage in detecting short action segments, which contain only a small number of clips and do not require much more aggregation of temporal features.

However, VideoMAEv2 covers a short range in the temporal dimension during training (no overlap sliding with a kernel size of 2); therefore, it captures less long-range information than does SlowFast, which uses both fast and slow branches for temporal feature extraction. When applied to TAD, VideoMAEv2 lacks the ability to detect long action segments.

In addition, temporal-level backbone networks often encounter the issue of imprecise alignment between their features and the corresponding spatial information at each instant. Instead, the features serve as a generalized representation within a particular temporal window. Moreover, the boundaries of certain actions are determined by the frames in which specific objects appear, thereby resulting in imprecise localization. To enhance the precision of localization, we introduce a spatial-level backbone into the framework.

\subsubsection{Large-scale spatial-level backbone.}
\label{sec:inst}
The aim of the spatial-level backbone is to enhance the localization of action boundaries that exhibit a strong correlation with the frame context, which is ignored by the temporal-level backbone.

To incorporate frame context into the temporal feature sequence, we propose a straightforward yet comprehensive feature pyramid that involves spatial-level visual methods in the detection process, aiding in precise localization. 

Concretely, we adopt the backbone of DINOv2~\citep{oquab2023dinov2} pretrained on the LVD-142M dataset and extract the output feature as the spatial representation for each instant. Then, we simply sample the spatial-level sequence with the same frame rate as the temporal-level sequence to match their sequence lengths. Next, we build feature pyramids with the SGP layer (w=1 for spatial-level backbone) and pooling for the temporal-level backbone and spatial-level backbone, respectively. 
Then, the temporal-level feature pyramid is directly fed into the classification head, while both the temporal-level feature pyramid and spatial-level feature pyramid are combined through element-by-element summation along each level of the pyramid and fed
into the Trident-head.

\begin{table*}[t]
\centering{
\caption{Comparison with the state-of-the-art methods on the HACS dataset.}
\label{table:hacs}
\setlength{\tabcolsep}{1.3mm}
\renewcommand{\arraystretch}{1.1}
\begin{tabular}{c | c |c |c  c c c c }
\toprule
Method & Venue & Backbone & 0.5 & 0.75 & 0.95 & Avg. \\
\midrule
SSN~\citep{SSN2017ICCV} & ICCV'2017& I3D & 28.8 & 18.8 & 5.3 &  19.0\\
LoFi~\citep{xu2021low} & NeurIPS'2021 & TSM & 37.8 & 24.4 & 7.3 & 24.6\\
G-TAD~\citep{xu2020g}& CVPR'2020 & I3D & 41.1 & 27.6 & 8.3 &  27.5\\
TadTR~\citep{liu2022end}& TIP'2022 & I3D & 47.1 & 32.1 & 10.9 & 32.1\\
BMN~\citep{lin2019bmn}& ICCV'2019 & SlowFast &  52.5 & 36.4 & 10.4 & 35.8\\
TALLFormer~\citep{cheng2022tallformer}& ECCV'2022 & Swin &  55.0 & 36.1 & 11.8 & 36.5\\
TCANet~\citep{qing2021temporal} & CVPR'2021 & SlowFast &  54.1 & 37.2 & 11.3 & 36.8\\
TCANet + BMN & CVPR'2021 & SlowFast & 55.6 & 40.0 & 11.5 & 38.7\\
ETAD~\citep{liu2023etad} & CVPRW'2023 & SlowFast & 55.7 & 39.0 & 13.8 & 38.8 \\
\midrule
\textbf{TriDet} & CVPR'2023 & I3D & 54.5 & 36.8 & 11.5 & 36.8\\
\textbf{TriDet} & CVPR'2023 & SlowFast  & 56.7 & 39.3 & 11.7 & 38.6\\
\textbf{TriDet} & - & DINOv2 & 52.4 & 33.6 & 8.4 & 33.7\\
\textbf{TriDet} & - & VideoMAEv2  & 62.4 & 44.1 & 13.1 & 43.1\\
\textbf{TriDet} & - & Fused &\textbf{ 63.0} & \textbf{44.5} & \textbf{12.9} & \textbf{43.4}\\
\bottomrule
\end{tabular}}
\end{table*}

Training and Inference
Each layer $l$ of the temporal-level feature pyramid or combined feature pyramid outputs a temporal feature $F^l \in \mc{R}^{(2^{l-1}T)\times D}$, which is then fed to the classification head or the Trident-head for action instance detection. The output of each instant $t$ in feature pyramid layer $l$ is denoted as $\hat{o}_{t}^l = (\hat{c}_{t}^l, \hat{d}_{st}^l, \hat{d}_{et}^l)$. 
The overall loss function is then defined as follows:
\begin{equation}
\begin{aligned}
\mc{L}&=\frac{1}{N_{pos}}\sum_{l,t}\mathbbm{1}_{\{c^l_t>0\}}(\sigma_{IoU}\mc{L}_{foc} + \mc{L}_{IoU})\\
&+ \frac{1}{N_{neg}}\sum_{l,t}\mathbbm{1}_{\{c^l_t=0\}}\mc{L}_{foc},
\end{aligned}
\end{equation}
where $\sigma_{IoU}$ represents the temporal IoU between the predicted segment and the ground truth action instance and $\mc{L}_{foc}$ and $\mc{L}_{IoU}$ are the focal loss~\citep{lin2017focal} and IoU loss~\citep{rezatofighi2019generalized}, respectively. $N_{pos}$ and $N_{neg}$ denote the number of positive and negative samples and $c_l^t$ is the classification label (0 for background).
The term $\sigma_{IoU}$ is used to reweight the classification loss at each instant, such that instants with better regression (\ie~of higher quality) contribute more to the training. 
Following previous methods~\citep{tian2019fcos,zhang2020bridging,zhang2022actionformer}, center sampling is adopted to determine the positive samples. Namely, the instants around the center of an action instance are labeled as positive, and all others are considered negative.

\myPara{Inference.} At inference time, the instants with classification scores higher than threshold $\lambda$ and their corresponding instances are kept. Last, Soft-NMS~\citep{bodla2017soft} is applied for the deduplication of predicted instances.

\section{Experiments}
\subsection{Datasets}
We conduct experiments on six challenging datasets, including two TAD task datasets: THUMOS14~\citep{THUMOS14} and HACS-Segment~\citep{zhao2019hacs} and a multilabel TAD task dataset MultiTHUMOS~\citep{yeung2015every}.

\myPara{TAD task datasets.} THUMOS14 consists of 20 sport action classes and contains 200 and 213 untrimmed videos with 3,007 and 3,358 action segments in the training set and test set, respectively. HACS is a large-scale dataset that contains 200 classes of action, which has 37,613 videos for training, as well as 5,981 videos for test.  

\myPara{Multilabel TAD task datasets.} MultiTHUMOS is a multilabel dataset that shares the same videos with the THUMOS14 dataset. It contains $38,690$ annotations of 65 action categories with an average of 1.5 labels per frame and 10.5 action classes per video.

\myPara{Evaluation.} For all these datasets, only the annotations of the training and validation sets are accessible. Following previous practice~\citep{lin2019bmn,zhang2022actionformer,cheng2022tallformer,zeng2019graph}, we evaluate on the validation set. We report the mean average precision (mAP) at different intersection over union (IoU) thresholds. For THUMOS14, we report IoU thresholds at [0.3:0.7:0.1]. For HACS, we report the result at the IoU threshold [0.5, 0.75, 0.95], and the average mAP is computed at [0.5:0.95:0.05]. For the multilabel datasets, we evaluate with detection-mAP instead of segmentation-mAP following the previous works~\citep{tan2022pointtad,tang2023temporalmaxer,shao2023action}. We report the average IoU with thresholds [0.1: 0.1: 0.9] for the MultiTHUMOS dataset.

\subsection{Implementation Details}
TriDet is trained end-to-end with the AdamW~\citep{loshchilov2018decoupled}
optimizer. The initial learning rate is set to $10^{-4}$ for THUMOS14 and MultiTHUMOS, and  $10^{-3}$ for HACS. We detach the gradient before the start boundary head and end boundary head and initialize the CNN weights of these two heads with a Gaussian distribution $\mc{N}(0, 0.1)$ to stabilize the training process. The learning rate is updated with the cosine annealing schedule~\citep{loshchilov2017sgdr}. We train for 40, 48, and 13 epochs for THUMOS14, MultiTHUMOS and HACS (including a warmup of 20, 20 and 10 epochs).

For HACS, the number of bins $B$ of the Trident-head is set to 14, the convolution window $w$ is set to 11, and the scale factor $k$ is set to 1.0. For THUMOS14, MultiTHUMOS the number of bins $B$ of the Trident-head is set to 16,  the convolution window $w$ is set to 1, and the scale factor $k$ is set to 1.5. We round the scaled window size and take it up to the nearest odd number for convenience. 

\subsection{Comparison with State-of-the-art Results}
\subsubsection{Single-label temporal action detection}
\myPara{HACS.} 
In our experiment, we utilized the largest dataset, HACS, and adopted three commonly used temporal-level backbones I3D, SlowFast, and VideoMAEv2 and spatial-level backbone DINOv2 to ensure a fair and comprehensive comparison. As shown in \tabref{table:hacs}, TriDet with VideoMAEv2 significantly outperforms the previous best method, with an average mAP margin of approximately 4.4\%. 
This improvement is attributed to the fact that insufficient learning during training leads to inadequate discriminative temporal features, yet VideoMAEv2 effectively addresses this problem by extensively pretraining on a large dataset. This makes the learned features more distinguishable, which greatly benefits the TAD task. 

In addition, through the fusion of the spatial-level backbone with DINOv2, the TriDet with fused FPN (denoted as TriDet-Fused) achieves a higher average mAP (43.4\%) than does VideoMAEv2 alone (43.1\%). This indicates that the spatial-level can further enhance the detection performance. We further analyze the results of the temporal-level backbone and spatial-level backbone in \secref{sec:dino_vmae}.

 \begin{table*}[t]
\centering{
\caption{Comparison with state-of-the-art methods on THUMOS14 dataset. *: reported in VideoMAEv2 paper~\citep{wang2023videomaev2}}
\label{table:thumos14}
\setlength{\tabcolsep}{1.3mm}
\begin{tabular}{c|c|c|c c c c c c }
\toprule
Method & Venue & Backbone & 0.3 & 0.4 & 0.5 & 0.6 & 0.7 & Avg.\\
\midrule
 BMN~\citep{lin2019bmn} & ICCV'2019 & TSN & 56.0 & 47.4 & 38.8 & 29.7 & 20.5 & 38.5 \\
 G-TAD~\citep{xu2020g} & CVPR'2020 & TSN  & 54.5 & 47.6 & 40.3 & 30.8 & 23.4 & 39.3 \\
 A2Net~\citep{yang2020revisiting}& TIP'2020 &I3D & 58.6 & 54.1 & 45.5 & 32.5 & 17.2 & 41.6 \\
 TCANet~\citep{qing2021temporal}& CVPR'2021 & TSN & 60.6 & 53.2 & 44.6 & 36.8 & 26.7 & 44.3 \\
 RTD-Net~\citep{tan2021relaxed}& ICCV'2021 &I3D & 68.3 & 62.3 & 51.9 & 38.8 & 23.7 & 49.0 \\
 VSGN~\citep{zhao2021video}& ICCV'2021 & TSN & 66.7 & 60.4 & 52.4 & 41.0 & 30.4 & 50.2 \\
 ContextLoc~\citep{zhu2021enriching}& ICCV'2021 &I3D & 68.3 & 63.8 & 54.3 & 41.8 & 26.2 & 50.9\\
AFSD~\citep{lin2021learning}& CVPR'2021 &I3D & 67.3 & 62.4 & 55.5 & 43.7 & 31.1 & 52.0 \\
ReAct~\citep{shi2022react}& ECCV'2022& TSN & 69.2 & 65.0 & 57.1 & 47.8 & 35.6 & 55.0 \\
TadTR~\citep{liu2022end}& TIP'2022 &I3D & 74.8 & 69.1 & 60.1 & 46.6 & 32.8 & 56.7\\
TALLFormer~\citep{cheng2022tallformer} & ECCV'2022 &Swin & 76.0 & - & 63.2 & - & 34.5 & 59.2\\
ActionFormer~\citep{zhang2022actionformer} & ECCV'2022 &I3D & 82.1 & 77.8 & 71.0 & 59.4 & 43.9 & 66.8\\
ActionFormer* & CVPR'2023 &VideoMAEv2 & 84.0 & 79.6 & 73.0 & 63.5 & 47.7 & 69.6\\
 \midrule
\textbf{TriDet}& CVPR'2023&I3D & 83.6 & 80.1 & 72.9 & 62.4 & 47.4 & 69.3\\
\textbf{TriDet}& -& DINOv2 & 67.9 & 62.2 & 53.2 & 41.8 & 27.7 & 50.6\\
\textbf{TriDet}& -&VideoMAEv2 & 84.8 & 80.0 & 73.3 & \textbf{63.8} & 48.8 & 70.1\\
\textbf{TriDet}& -& Fused & \textbf{85.5} & \textbf{80.7} & \textbf{73.9} & 62.9 & \textbf{\textbf{48.9}} & \textbf{70.4}\\
\bottomrule
\end{tabular}
}
\end{table*}

\myPara{THUMOS14.} We conduct the same comparison as on HACS. As \tabref{table:thumos14} shows, TriDet also achieved state-of-the-art performance using the I3D backbone, with an average mAP of up to 69.3\%, demonstrating the effectiveness of TriDet. VideoMAEv2 improves the average mAP by 0.7\%, and TriDet-Fused further boosts the performance of TriDet, reaching an average mAP of 70.4\%. These results are consistent with the results on HACS.

\begin{table*}[t]
\centering{
\caption{Comparison with the state-of-the-art methods on the MultiTHUMOS datasets.}
\label{table:multilabel}
\setlength{\tabcolsep}{1.6mm}
\renewcommand{\arraystretch}{1.1}
\begin{tabular}{c|c|c|c|c|cccc}
\toprule
\multirow{2}{*}{Method} & \multirow{2}{*}{Venue} & \multirow{2}{*}{Backbone} & \multirow{2}{*}{RGB} & \multirow{2}{*}{Flow} & \multicolumn{4}{c}{MultiTHUMOS} \\
\cline{6-9}
                        &                           &                      &                       & & 0.2    & 0.5    & 0.7   & Avg    \\
\midrule
PointTAD~\citep{tan2022pointtad} &NeurIPS’2022               & I3D                       &                      $\surd$&                       & 39.7   & 24.9   & 12.0  & 23.5  \\
ASL~\citep{shao2023action}&Arxiv’2023  & I3D                       &                      $\surd$&                       & 42.4   & 27.8   & 13.7  & 25.5 \\
TemporalMaxer~\citep{tang2023temporalmaxer}&Arxiv’2023           & I3D                       &                      $\surd$&                       & 47.5   & 33.4   & 17.4  & 29.9  \\
\midrule
\textbf{TriDet}                  & - & I3D                       &                      $\surd$&                       & 49.1   & 34.3   & 17.8   & 30.7  \\
\textbf{TriDet}                 & -  & I3D                       &                      $\surd$&$\surd$                       & 55.7   & 41.0   & 23.5   & 36.2  \\
\textbf{TriDet}                 & -  & VideoMAEv2                &                      $\surd$&                       & 57.7   & 42.7   & 24.3   & 37.5  \\
\textbf{Tridet}           & -  & Fused                &                      $\surd$&                       & 57.6   & \textbf{42.9}   & \textbf{25.0}   & \textbf{37.7}   \\
\bottomrule
\end{tabular}
}
\end{table*}

\subsubsection{Multilabel temporal action detection}
For the multilabel temporal action detection task, multiple levels of labels may exist at each instant. 
As \tabref{table:multilabel} shows, our method achieves impressive performance on the multilabel task and significantly surpasses all previous methods on the MultiTHUMOS. 
In addition, we perform supplementary experiments on MultiTHUMOS to assess the influence of two-stream backbones (I3D with RGB and optical flow as input) and larger models (VideoMAEv2 + DINOv2). Clearly, incorporating both optical flow backbones and large model backbones results in noteworthy enhancement. Both optical flow and VideoMAEv2 capture the motion dynamics across consecutive frames, and the findings demonstrate that utilizing motion dynamics can effectively enhance the detection performance in the multilabel TAD task.

\begin{table}[t]
\centering{
\caption{Analysis of the effectiveness of three main components on THUMOS14.}
\label{table:ablation}
\setlength{\tabcolsep}{1.0mm}
\renewcommand{\arraystretch}{1.2}
\scalebox{1}{
\begin{tabular}{c| c  c  c| c c c c}
\toprule
Method & SA & SGP & Trident & 0.3 & 0.5 & 0.7 & Avg. \\
\midrule
1 & &  &  &  77.3 & 65.2 & 40.0 & 62.1 \\
2 & $\surd$  &  & & 82.1 & 71.0 & 43.9 & 66.8 \\
3 & & $\surd$ & & \textbf{83.6} & 71.7 & 45.8 & 68.3 \\
4 & & $\surd$ & $\surd$  &\textbf{ 83.6}  & \textbf{72.9} & \textbf{47.4} & \textbf{69.3} \\
\bottomrule
\end{tabular}
}}
\end{table}

\subsection{Ablation on the Main Components}
We demonstrate the effectiveness of our proposed components in TriDet: SGP layer and Trident-head. To verify the effectiveness of our SGP layer, with the I3D backbone, we use a baseline feature pyramid used by~\citep{lin2021learning,zhang2022actionformer} to replace our SGP layer. The baseline consists of two 1D-convolutional layers and a shortcut. The window size of the convolutional layers is set to 3, and the number of channels of the intermediate features is set to the same dimension as the intermediate dimension in the FFN in our SGP layer. All other hyperparameters (\eg~number of the pyramid layers, etc.) are set the same as in our TriDet.

As depicted in \tabref{table:ablation}, compared with the baseline model we implement (Row 1), the SGP layer brings a $6.2\%$ absolute improvement in the average mAP. Second, we compare the SGP with the previous state-of-the-art method, ActionFormer, which adopts the self-attention mechanism in a sliding window \citep{beltagy2020longformer} with window size $7$ (Row 2). Our SGP layer still achieves a $1.5\%$ improvement in average mAP, demonstrating that the convolutional network has excellent performance in TAD. Furthermore, we compare our Trident-head with the normal spatial-level regression head, which regresses the boundary distance for each instant. Trident-head improves the average mAP by $1.0\%$, and the mAP improvement is more obvious in the case of high IoU threshold (\eg~$1.6\%$ average mAP improvement in IoU 0.7\%). 

\begin{figure}[t!]
    \centering{
    \includegraphics[width=0.9\linewidth]{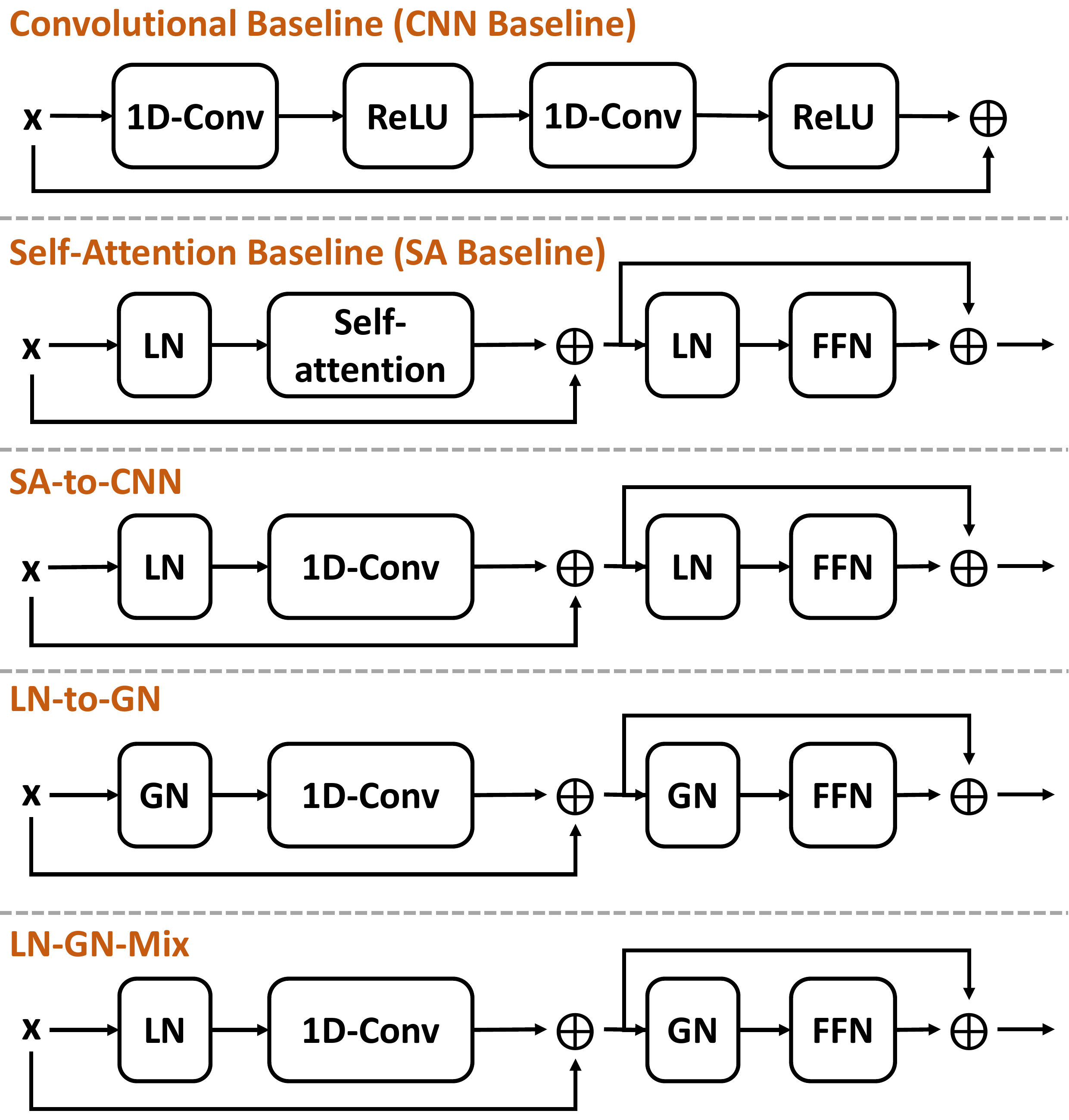}
    
  \caption{Two baseline models and three different variants of the convolutional-based structure. }
  \label{fig:self2cnn} 
  }
\end{figure}

\subsection{The Core Effectiveness of Transformer}
\label{sec:transformer}
As described in \secref{sec:introduction}, we claim that the previous best transformer-based method~\citep{zhang2022actionformer} primarily improved due to the macro-architecture of the transformer rather than the self-attention mechanism. To be self-contained, in this section, we further analyze the impact of module design on the detector.

\myPara{From Transformer to CNN.} As \figref{fig:self2cnn} shows, for comparison, we build two baseline models: a convolutional (CNN) baseline and a self-attention (SA) baseline. First, we build a CNN baseline in which the convolutional module is adopted from the previous one-stage detector~\citep{lin2021learning,zhang2022actionformer}. Second, the previous state-of-the-art detector~\citep{zhang2022actionformer} with local window self-attention~\citep{beltagy2020longformer} is chosen as the SA baseline. Then, to analyze the importance of two common components: self-attention and normalization, in the Transformer~\citep{vaswani2017attention} macro-architecture, we provide three variants of the convolutional-based structure: SA-to-CNN, LN-to-GN and LN-GN-Mix, and validate their performance on THUMOS14 with the I3D backbone. 

To verify the robustness and sensitivity of Trident-head across different architectures of transformer, we conducted additional comparisons of the three aforementioned variants with and without the inclusion of Trident-head.

\begin{table}[t]
\centering{
\caption{The results of different variants on THUMOS14. *: with Trident-head.}
\label{table:self2cnn}
\begin{tabular}{c|c c c c }
\toprule
Method & 0.3 & 0.5  & 0.7 & Avg.\\
\midrule
 CNN Baseline & 77.3  & 65.2  & 40.0 & 62.1 \\
 SA Baseline & 82.1  & 71.0  & 43.9 & 66.8 \\
 \midrule
 SA-to-CNN & 80.4  & 67.5  & 42.9 & 64.9 \\
 LN-to-GN & 80.0  & 68.0  & 42.3 & 64.8 \\
 LN-GN-Mix & 80.8  & 68.8  & 43.6 & 65.7 \\
 \midrule
  SA-to-CNN* & 81.2  & 68.7  & 43.5 & 65.7 \\
 LN-to-GN* & 80.7  & 69.1  & 42.2 & 65.4 \\
 LN-GN-Mix* & 81.6  & 69.5 & 42.9 & 66.0 \\
\bottomrule
\end{tabular}
}
\end{table}

\myPara{Results.}
From ~\tabref{table:self2cnn}, we can see a large performance gap between the SA baseline and the CNN baseline (approximately $4.7\%$ in average mAP), demonstrating that the transformer holds a large advantage for TAD tasks. 
Then, we conduct an ablation study with the three variants of the normal regression head.

We first simply replace the local self-attention with a 1D convolutional layer (SA-to-CNN) that has the same receptive field as that of~\citep{zhang2022actionformer} (\eg~kernel size is 19). This change yields a dramatic performance increase in terms of the average mAP compared with the CNN baseline (approximately $2.8\%$) but is still behind the transformer baseline by approximately $1.9\%$. Next, we conduct experiments with different normalization layers (\ie~layer normalization (LN)~\citep{ba2016layer} and group normalization (GN)~\citep{wu2018group} (\ie~LN-to-GN and LN-GN-Mix) and find that the hybrid structure of LN and GN (LN-GN-Mix) shows better performance compared to the SA baseline ($65.7\%$ versus $64.9\%$). 

By combining with the Trident-head, the results consistently remain robust for all three variants (\ie~the performance relationship remains unchanged), demonstrating that the structural improvement remains reliable under different regression heads. Additionally, the LN-GN-Mix version achieves an average mAP of $66.0\%$, showcasing the potential for efficient convolutional modeling. Furthermore, convolution relaxes the restriction on the weight distributions at each moment (i.e., the sum of the weights is not necessarily 1), avoiding the rank-loss problem. Hence, these empirical results further motivate us to enhance the feature pyramid with the SGP layer.

\begin{figure}[t]
    \centering{
    \includegraphics[width=0.85\linewidth]{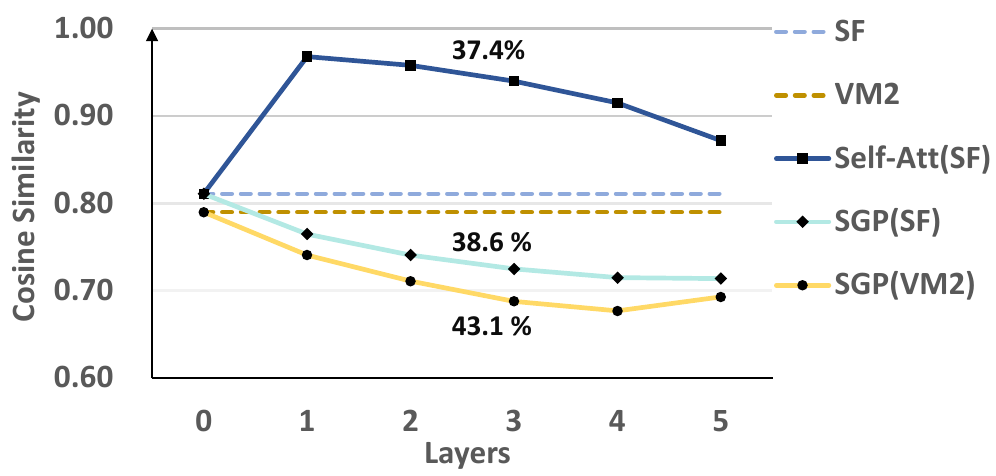}
    
  \caption{The statistical average cosine similarity and average mAP for self-attention (Self-Att) and SGP layer with SlowFast (SF) and VideoMAEv2 (VM2) backbones.
  }
  \label{fig:cosine} 
  }
\end{figure}

\subsection{Investigation of the Rank-Loss Problem}
\label{sec:rank_loss}
In this section, we validate the rank-loss problem in self-attention and the effectiveness of the SGP layer in preserving rank using the cosine similarity metric. The cosine similarity metric was chosen due to the fixed modulus property of the pure layer normalization.
\begin{property}[Fixed modulus property]
A pure Layer normalization $x'=LayerNorm(x)$ normalizes the temporal features $x\in\mc{R}^{n}$ to a fixed modulus $\sqrt{n-1}$
\end{property}
\begin{proof}
Consider the data in each dimension $x'_i$: 
\begin{equation}
\begin{aligned}
    x'_i=&\frac{x_i-mean(x)}{std(x)}\\
        =&\frac{x_i-mean(x)}{\sqrt{\frac{1}{n-1}\sum_{i\in n}{(x_i-mean(x))^2}}},
\end{aligned}    
\end{equation}
then we have 
\begin{equation}
\begin{aligned}
    ||x'||_2 = & \sqrt{\sum_{i\in n}{{x'}_i^{2}}}\\
             = & \sqrt{\frac{\sum_{i\in n}{(x_i-mean(x))^2}}{\frac{1}{n-1}\sum_{i\in n}{(x_i-mean(x))^2}}} \\
             = & \sqrt{n-1}.
\end{aligned}    
\end{equation}
\end{proof}

To simplify, we use cosine similarity  $S_c$ to measure the angular similarity between features at each
instant and the video-level average feature:
\begin{equation}
S_c=\frac{1}{T}\sum_{i\in T}{cos(x_i,\bar{x})}
\end{equation}
where $\bar{x}=\frac{1}{T}\sum_{i\in T}{x_i}$.

We assess the average cosine similarity for the self-attention and the SGP layer with SlowFast and VideoMAEv2 backbones on the HACS dataset and present the results in \figref{fig:cosine}. 

First, self-attention increases the similarity of the temporal feature sequence for each layer output, surpassing even the similarity of the temporal feature sequence directly predicted by the backbone. This suggests that TAD faces the rank-loss problem due to self-attention. Conversely, our SGP layer mitigates this issue and demonstrates more discriminative power (38.6\% vs 37.4\% in average mAP). Furthermore, the SGP layer presents similar characteristics across various backbone networks (\ie~SlowFast and VideoMAEv2), thereby highlighting its robustness.

Second, the VideoMAEv2 backbone generates features with lower cosine similarity compared to the SlowFast backbone. This suggests that visual models pretrained using abundant video data produce more distinct features in the temporal dimension.

\begin{table}[t]
\centering{
\caption{The effectiveness of different fusion policies and fusion position.}
\label{table:fusion}
\setlength{\tabcolsep}{1.1mm}
\renewcommand{\arraystretch}{1.3}
\begin{tabular}{|c|c|c|c|c|c|}
\hline
 & Policy           & mAP &  & Policy          & mAP \\ \hline 
\multirow{5}{*}{\shortstack{Fusion \\ policy\\(early fusion)}}                       & add              & 69.7       &       \multirow{5}{*}{\shortstack{Fusion\\ position\\(add)}}                  & early fusion & 69.7       \\ \cline{2-3} \cline{5-6} 
                       & concat           & 69.9       &                         & before SGP      & 69.9       \\ \cline{2-3} \cline{5-6} 
                       & cross-atten-T & 69.2       &                         & after SGP       & 70.6       \\ \cline{2-3} \cline{5-6} 
                       & cross-atten-S & 52.2       &                         & FPN decouple    & \textbf{71.2}       \\ \cline{2-3} \cline{5-6} 
                       & conv-atten       & \textbf{70.3}       &                         &                 &             \\ \hline
\end{tabular}}
\end{table}

\subsection{The Policies of Backbone Fusion}
\label{sec:fusion}
To explore an effective way to fuse features from temporal-level and spatial-level backbones, we conducted experiments using VideoMAEv2 and DINOv2 on THUMOS14, testing various fusion policies and fusion positions, and reporting their average mAP.

To recap, we decouple the FPN with the aim of improving localization precision and posit that spatial-level context is more beneficial for localization, while the classification of actions relies more on temporal-level context (motion information). To demonstrate the effectiveness of our decoupled fusion method, we conducted tests on the detection performance of multiple commonly used fusion methods, including both decoupled and coupled ones. By means of these tests, we illustrate the impact of these methods.

As shown in the left half of \tabref{table:fusion}, for the fusion policies, we report five typical methods and conduct the fusion after the embedding network (a convolutional layer is employed to embed the output feature from two backbone networks into the same dimension, resulting in \emph{early fusion}): (1) element-wise addition (\textit{add}); (2) concatenation (\textit{concat}); (3) local cross-attention (window size=5) with the temporal-level feature as value and spatial-level feature as query (\textit{cross-atten-T}); (4) local cross-attention (window size=5) with the spatial-level feature as value and temporal-level feature as query (\textit{cross-atten-S}); (5) predict an element-wise attention score with a convolutional layer from the concatenated features and weighted-sum of the two features (\textit{conv-atten}).

The \textit{conv-atten} policy outperforms other methods, but this method fails to fully utilize the spatial-level backbone features. One possible reason is that early fusion reduces the effectiveness of the spatial-level context, as it gets destroyed by convolution after being fed into the temporal-level feature pyramid. Therefore, we further study where fusion can achieve the best result.

Specifically, three positions are considered: (1) after the embedding network (\textit{early fusion}, which was mentioned in the previous paragraph; (2) before the SGP layer (\textit{before SGP}), we first downsample the spatial-level feature using max-pooling and then perform element-wise summation separately for the temporal-level feature and the spatial-level feature in each pyramid level; and (3) after the SGP layer (\textit{after SGP}), where the spatial-level features constructed by max-pooling are combined with the temporal features extracted by the SGP layer.

Upon analyzing \tabref{table:fusion} (right part), it was found that delaying fusion leads to a higher average mAP. This suggests that late fusion better preserves spatial-level context, thereby improving detection head prediction.

Based on this result, we further improve the \textit{after SGP} policy by decoupling the fusion process after the SGP layer, as detailed in \secref{sec:inst}. Decoupling the fusion improves the results significantly compared to those of other fusion methods (71.2\% vs 70.6\%). This demonstrates that separating the fusion process of the spatial-level and temporal-level backbone can greatly benefit the localization process, without causing any interference with the classification process.

\begin{figure}[t]
    \centering
    \subcaptionbox{TriDet (DINOv2)\label{fig:dino_sen}}{
        \includegraphics[width=\linewidth]{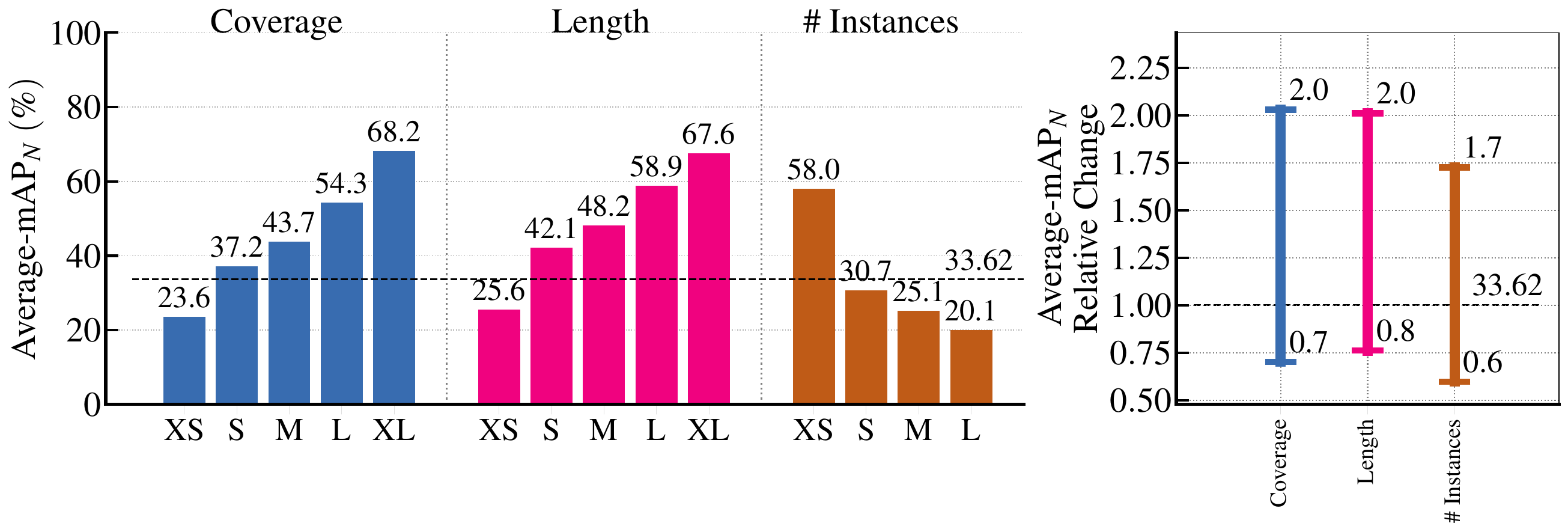}
    }
    \subcaptionbox{TriDet (SlowFast)\label{fig:sf_sen}}{
        \includegraphics[width=\linewidth]{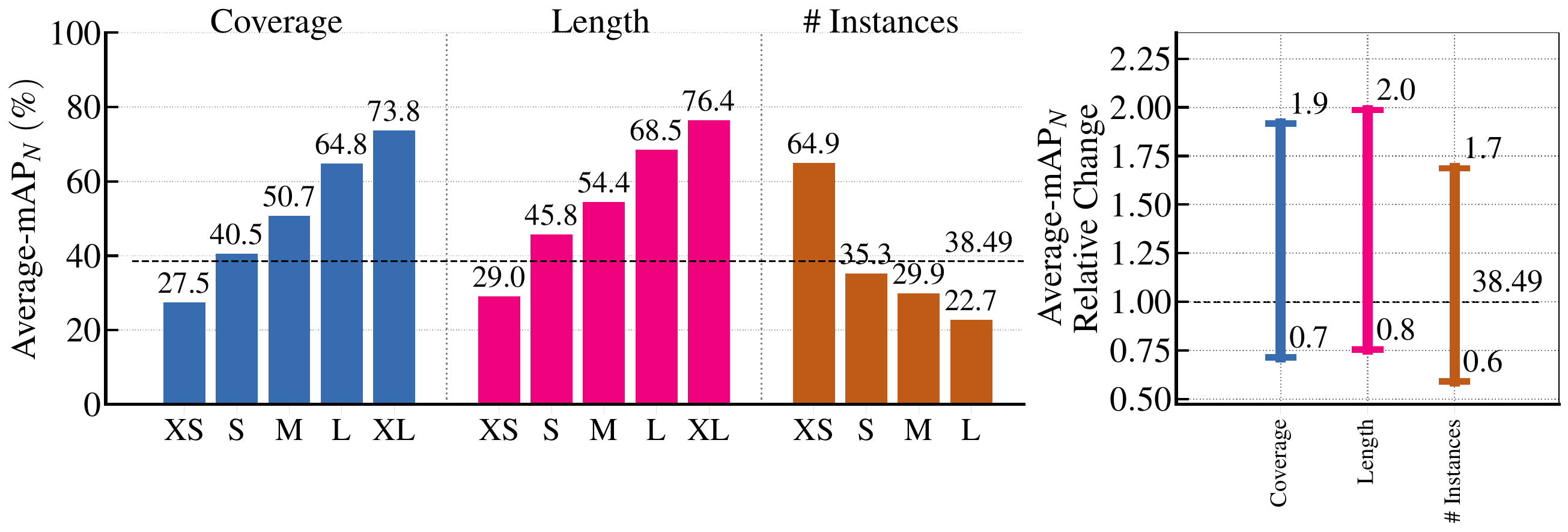}
    }
    
    \subcaptionbox{TriDet (VideoMAEv2)\label{fig:vmae_sen}}{
        \includegraphics[width=\linewidth]{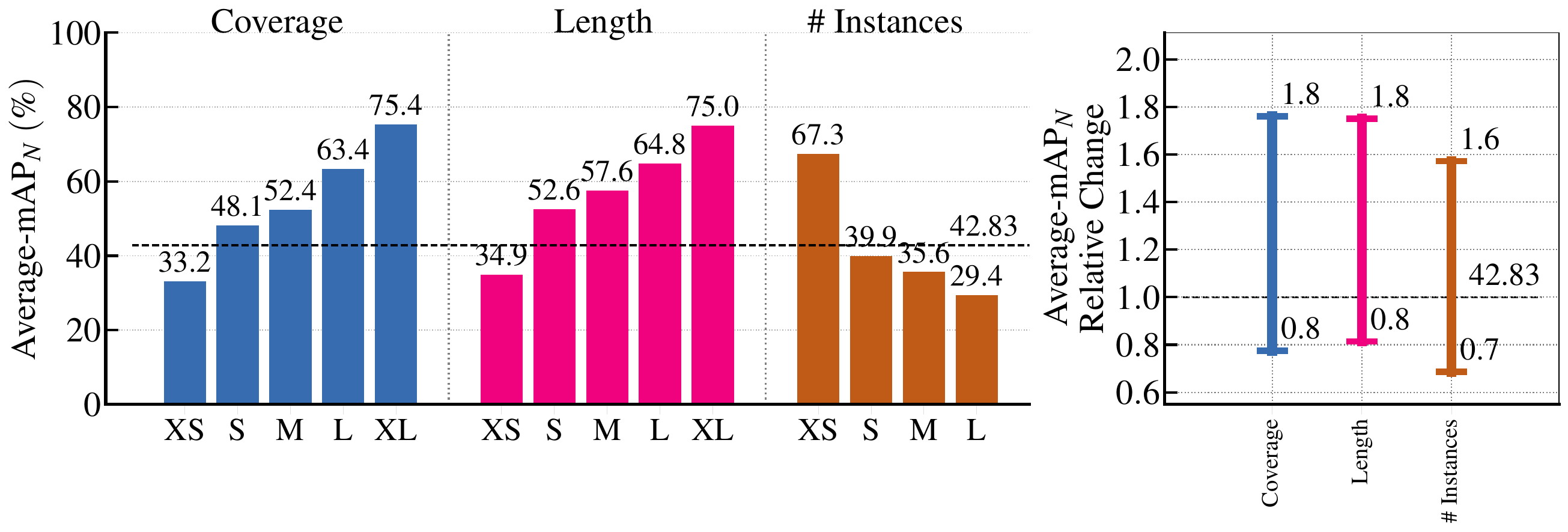}
    }
    \subcaptionbox{TriDet-Fused\label{fig:vmae_dino_sen}}{
        \includegraphics[width=\linewidth]{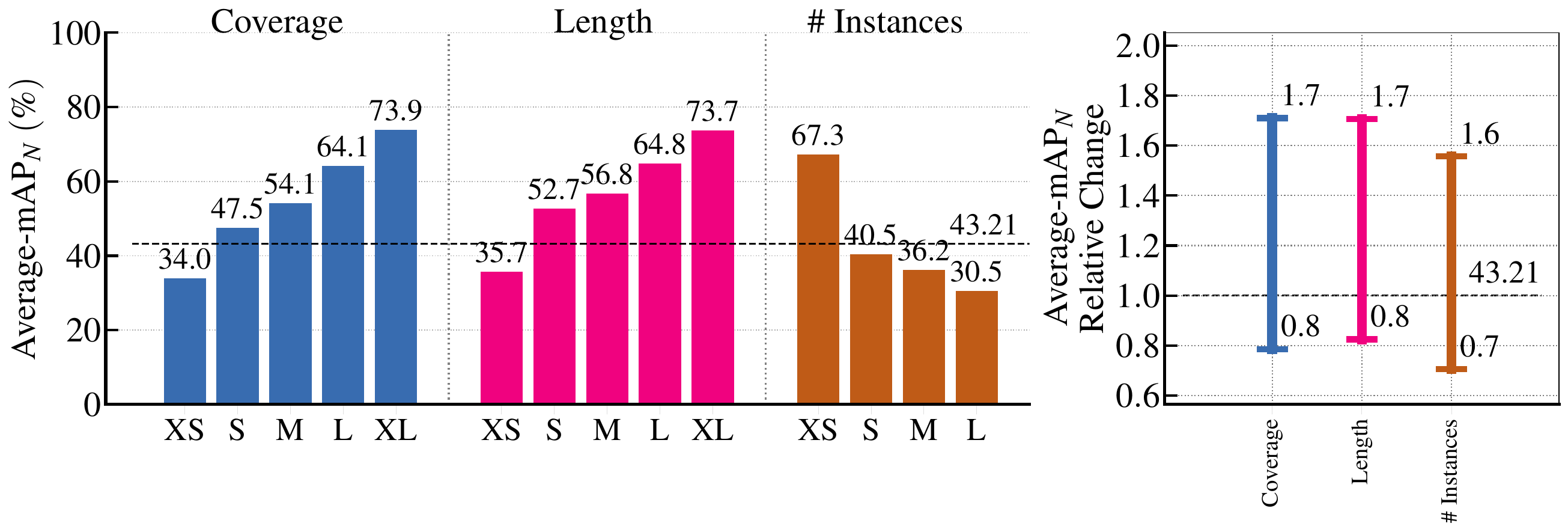}
    }

  \caption{Sensitivity analysis of the detection results, where $mAP_{N}$ is the normalized mAP with the average number $N$ of ground truth segments per class~\citep{alwassel2018diagnosing}.}
  \label{fig:sen}
\end{figure}

\begin{figure*}[t]
    \begin{minipage}[]{\textwidth}
    \centering
    \subcaptionbox{TriDet (DINOv2)\label{fig:dino_fn}}{
        \includegraphics[width=0.45\linewidth]{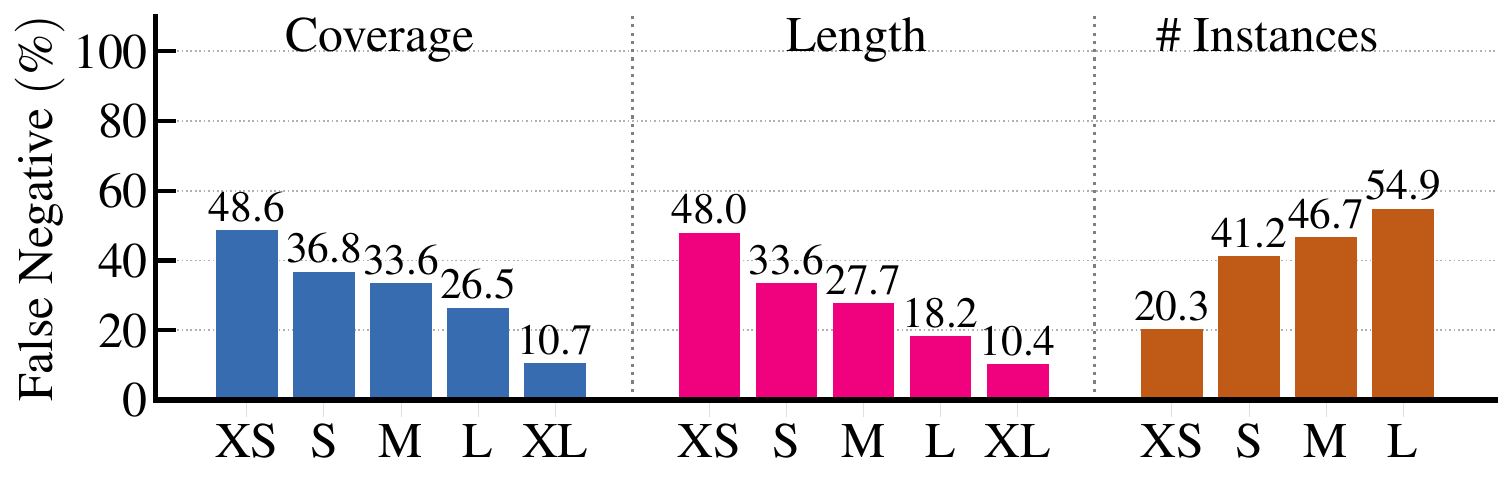}
    }
    \quad
    \subcaptionbox{TriDet (SlowFast)\label{fig:sf_fn}}{
        \includegraphics[width=0.45\linewidth]{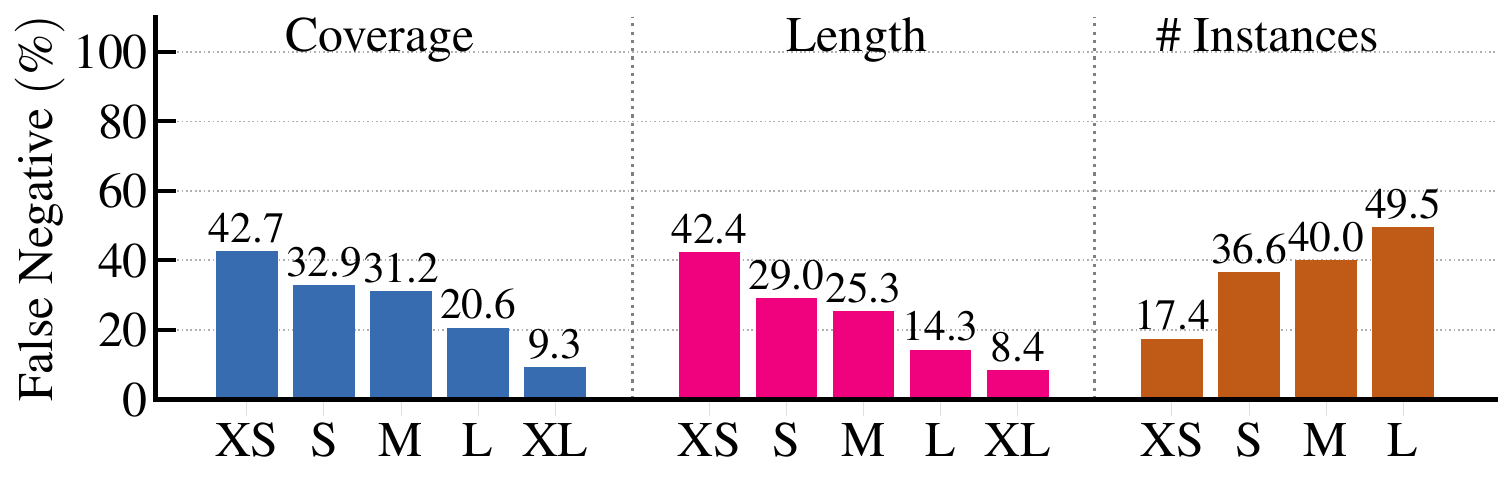}
    }
    
    \subcaptionbox{TriDet (VideoMAEv2)\label{fig:vmae_fn}}{
        \includegraphics[width=0.45\linewidth]{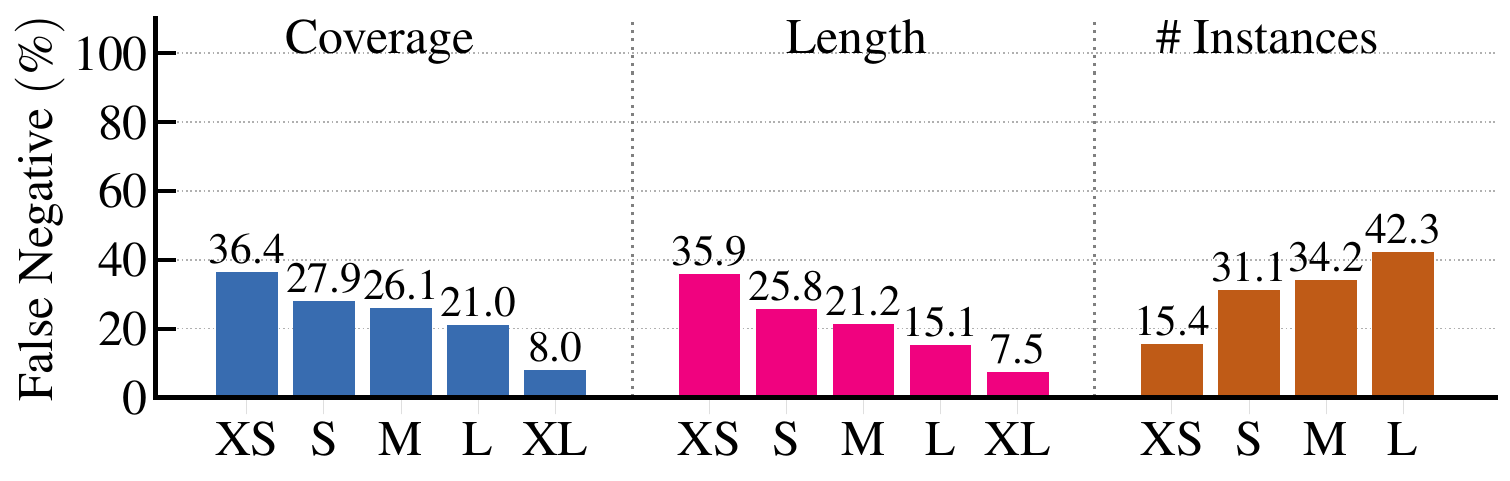}
    }
    \quad
    \subcaptionbox{TriDet-Fused\label{fig:vmae_dino_fn}}{
        \includegraphics[width=0.45\linewidth]{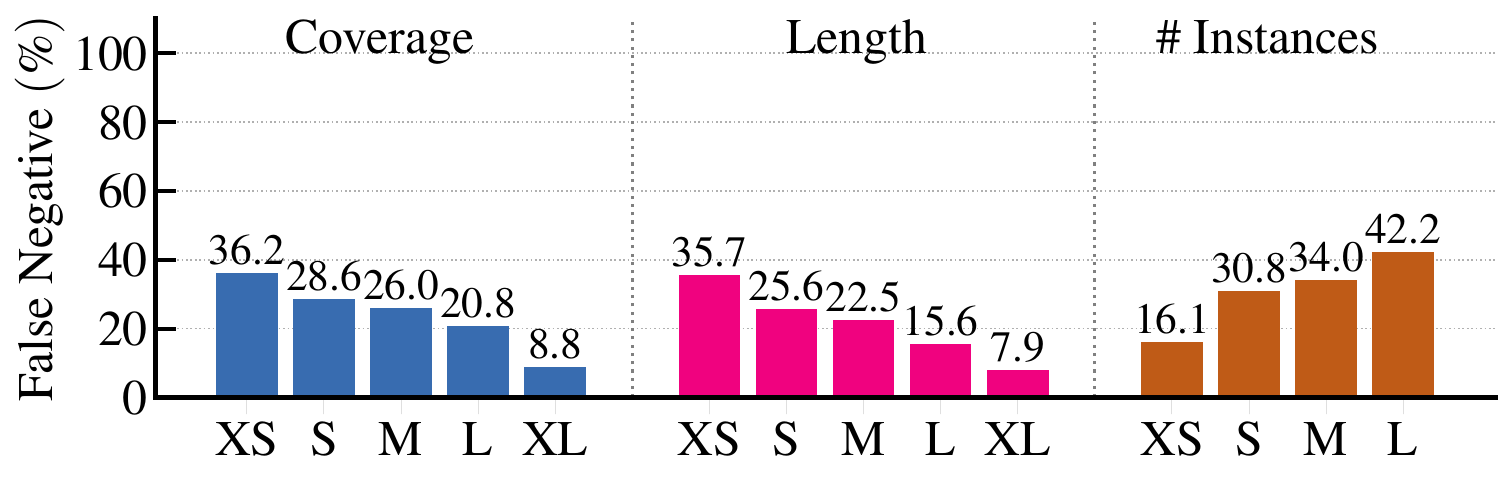}
    }

  \caption{The false negative analysis of the detection results, which counts the percentage of several common types of detection error in different Top-KG prediction groups, where G is the number of groundtruth segments.}
  \vspace{0.2cm}
  \label{fig:falsenegative}
  \end{minipage}
  \begin{minipage}[]{\textwidth}
  \centering
      \subcaptionbox{TriDet (DINOv2)\label{fig:dino_fp}}{
        \includegraphics[width=0.45\linewidth]{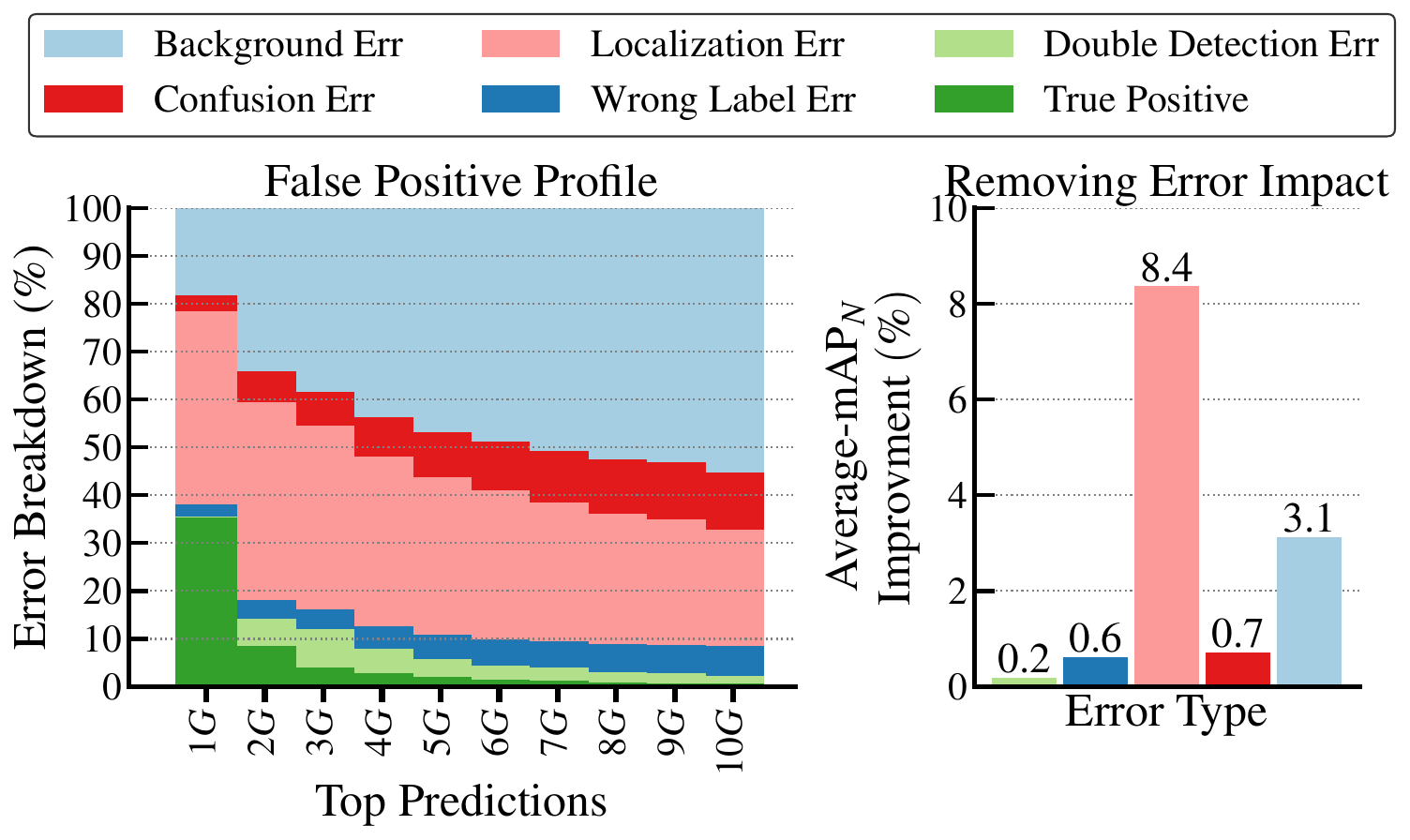}
    }
    \quad
    \subcaptionbox{TriDet (SlowFast)\label{fig:sf_fp}}{
        \includegraphics[width=0.45\linewidth]{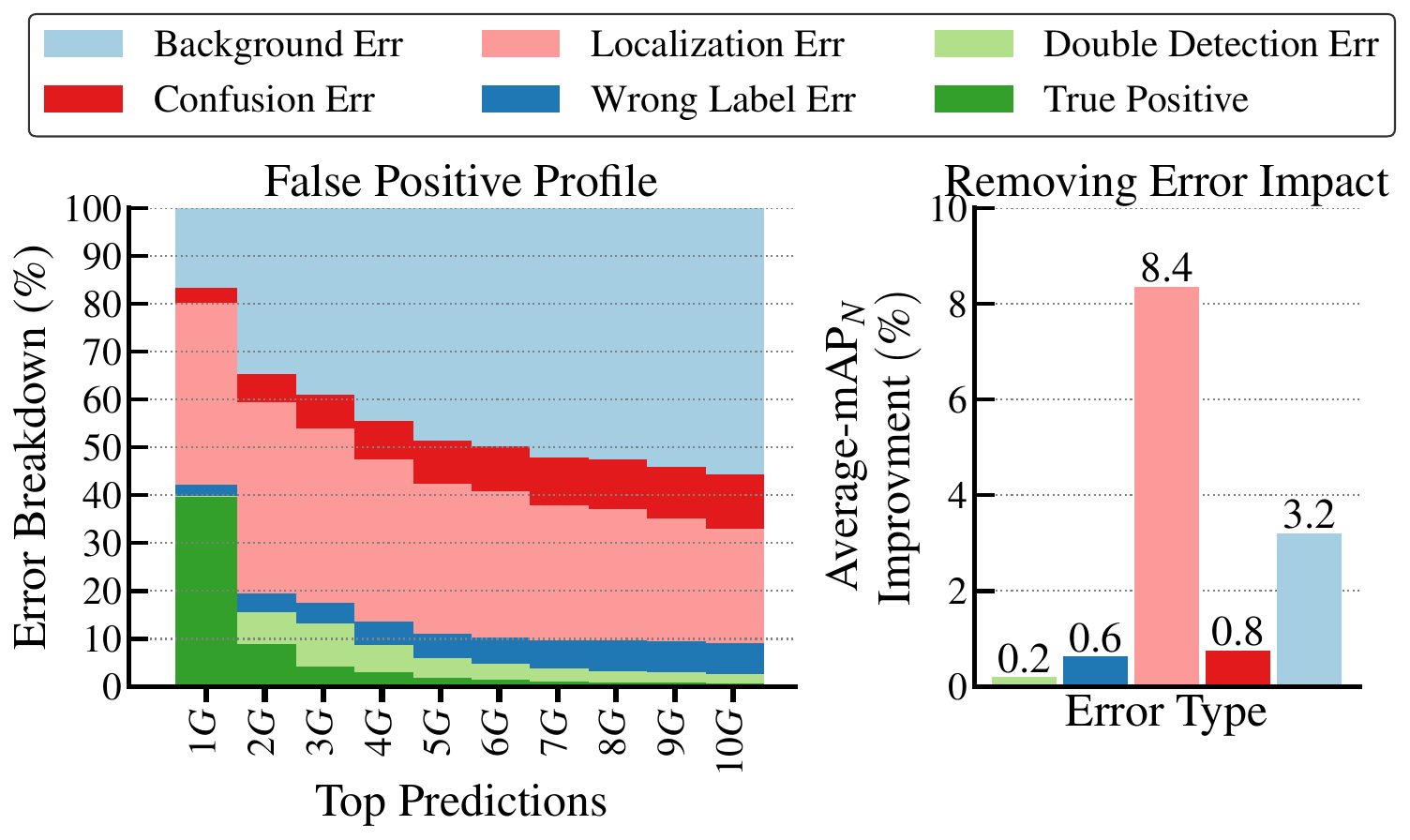}
    }
    
    \subcaptionbox{TriDet (VideoMAEv2)\label{fig:vmae_fp}}{
        \includegraphics[width=0.45\linewidth]{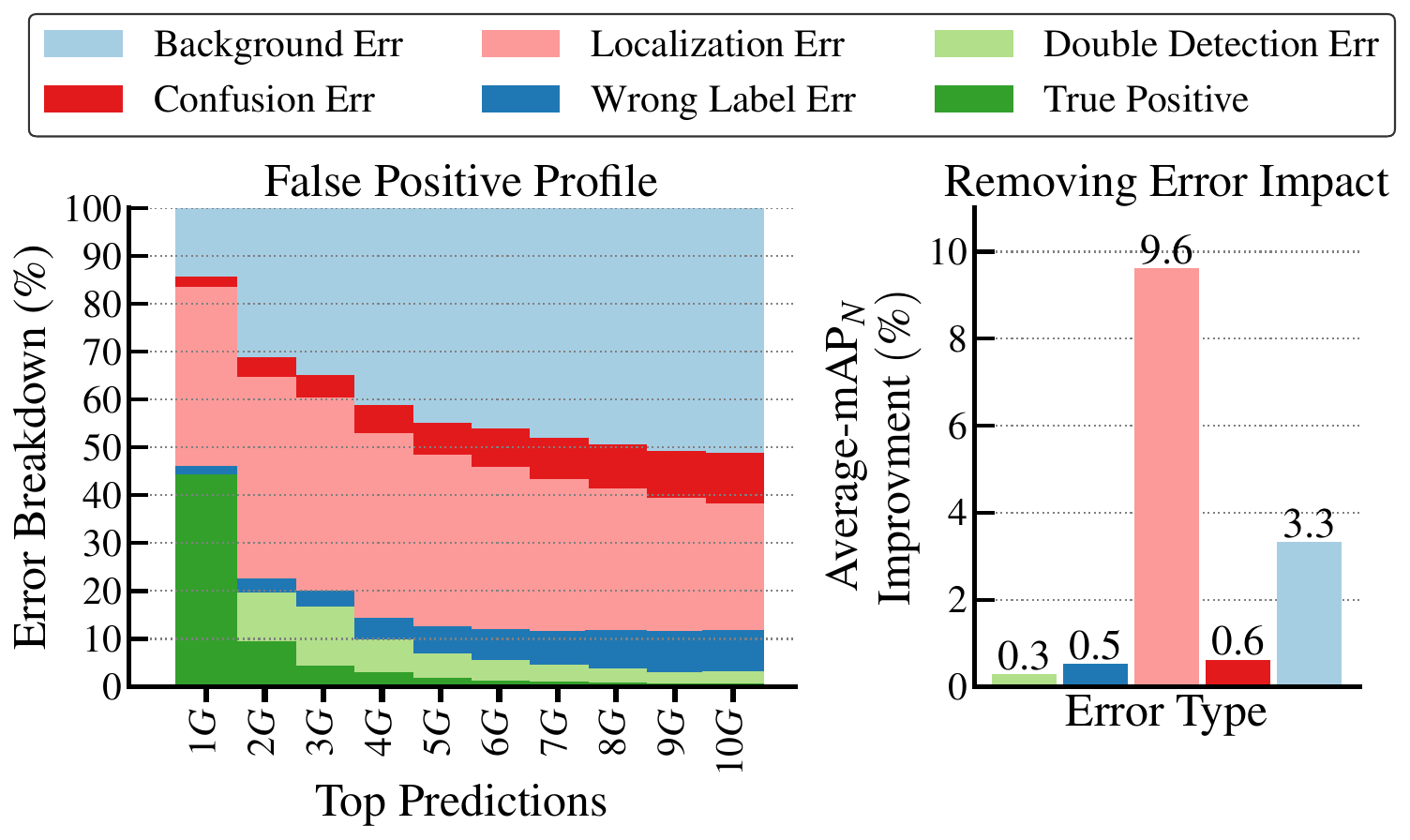}
    }
    \quad
    \subcaptionbox{TriDet-Fused\label{fig:vmae_dino_fp}}{
        \includegraphics[width=0.45\linewidth]{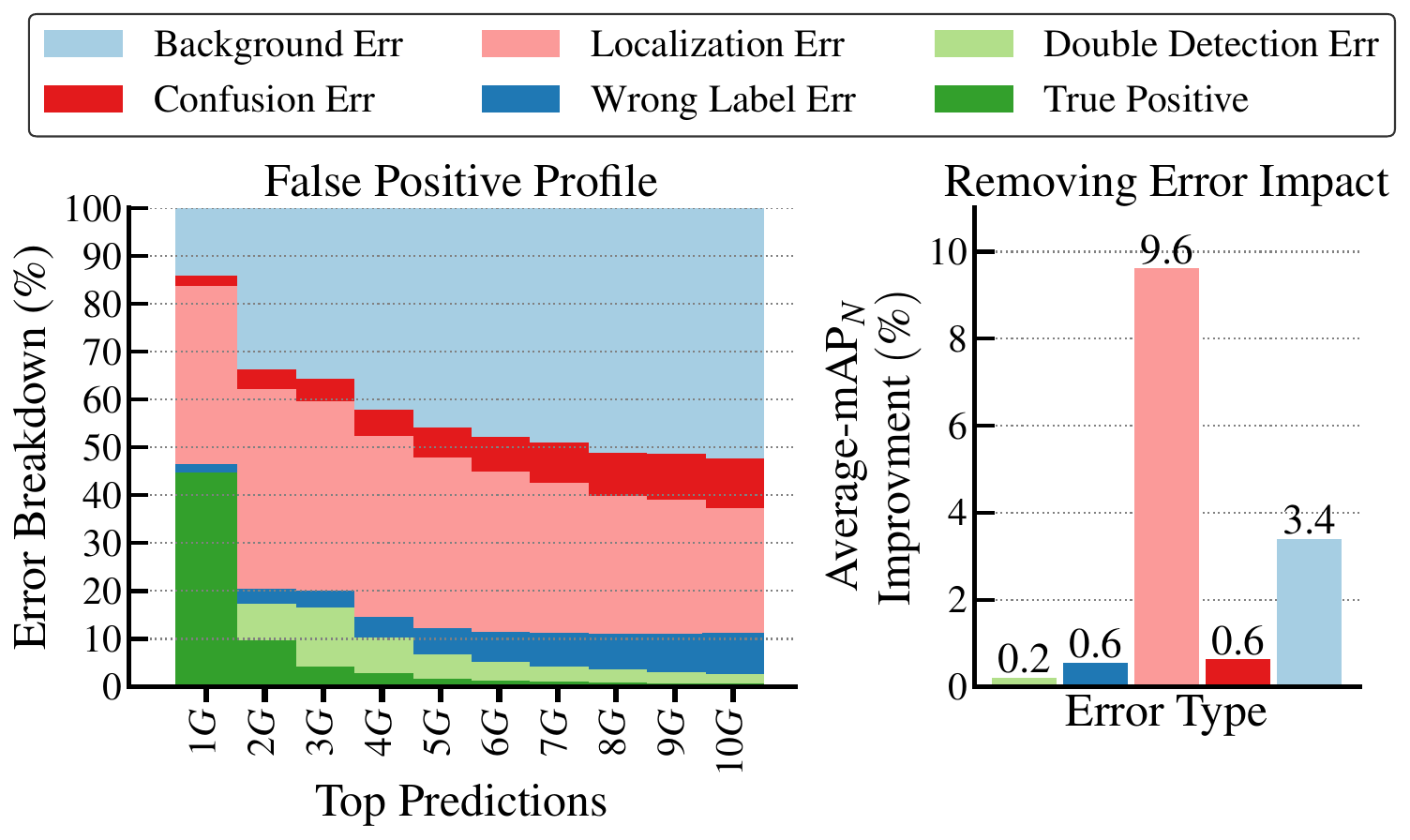}
    }

  \caption{False positive analysis of the detection results, which counts the percentage of several common types of detection error in different Top-KG prediction groups, where G is the number of groundtruth segments.}
  \label{fig:falsepositive}
  \vspace{-0.1cm}
  \end{minipage}
\end{figure*}

\subsection{Error Analysis for Different Backbones}
In this section, we utilize the tool provided by ~\citep{alwassel2018diagnosing} to analyze the detection results for the four types of backbones on HACS, including DINOv2, SlowFast, VideoMAEv2 and the fusion of VideoMAEv2 and DINOv2. Specifically, we analyze the false positives, false negatives, and sensitivity.

The false positive analysis counts the percentage of five types of detection error in different Top-KG predictions, where G is the number of groundtruth segments. The five types of detection error are (1) background error (\ie~background segments are predicted as action instances), (2) localization error (\ie~an instance that correctly predicts the label and $0.1 \leq \sigma_{IoU} < \alpha$, where $\alpha$ is the preset threshold), (3) double detection error (\ie~accurate but repeatedly predicted instances), (4) confusion error (\ie~an instance that incorrectly predicts the label and $0.1 \leq \sigma_{IoU} < \alpha$) and (5) {wrong label error (\ie~an instance that incorrectly predicts the label but $\sigma_{IoU} \geq \alpha$).
In addition, the removing error impact is the improvement gained from removing the error predictions of different types.

The false negative analysis and sensitivity analysis are conducted for varying characteristics.
Here, characteristics include coverage (\ie~normalized length), length (\ie~absolute length), and number of instances. The tool divides the test data into groups (denoted as XS, S, M, L, XL) based on \emph{instances of different lengths} and \emph{videos with different numbers of instances} and analyzes the results of each group individually. For average, the five groups are (0,0.2], (0.2,0.4], (0.4,0.6], (0.6,0.8], (0.8,1]. For length, the five groups are (0,30], (30,60], (60,120], (120,180], (180,$+\inf$) in seconds. For the number of instances, the four groups are (0,1], (1, 4], (4,8], (8,$+inf$). 

\subsubsection{Comparison of VideoMAEv2 and SlowFast}
\label{sec:vmae_slowfast}
In this section, we conduct experiments to analyse the pros and cons for VideoMAEv2 in \secref{sec:videomae} comparing with SlowFast. SlowFast is a powerful backbone that contains slow and fast paths to capture context at different scales, and is trained on the Kinect dataset. VideoMAEv2 is a temporal-level backbone that captures temporal context with a short scale (window size = 16), and is trained on a large amount of unlabeled video data and fine-tuned on the Kinect dataset.

In \figref{fig:sf_sen} and \figref{fig:vmae_sen}, we observe that VideoMAEv2 has a significant improvement in accuracy for short action segments (XS, S, M in the pink bar). However, in the case of long action segments (L and XL in the pink bar), SlowFast outperforms VideoMAEv2. It suggests that the long temporal context captured by the SlowFast is beneficial for detecting long action segments. On the other hand, since HACS contains a larger proportion of short instances, VideoMAEv2 achieves better overall results.

In \figref{fig:sf_fn} and \figref{fig:vmae_fn}, the false negative rate of VideoMAEv2 is lower than SlowFast in all characteristics. However, in \figref{fig:sf_fp} and \figref{fig:vmae_fp}, we also observe an increase in the removing error impact of \emph{localization error} and \emph{background error} for VideoMAEv2 compared to SlowFast. 
The results above indicate that VideoMAEv2 enhances the overall mAP by increasing the detection rate of action segments. However, it also results in more redundant and imprecise predictions.

\subsubsection{Comparison for DINOv2 and VideoMAEv2}
\label{sec:dino_vmae}
In this section, we investigate the effectiveness and differences of temporal-level and spatial-level pre-trained backbones with  VideoMAEv2 and DINOv2 in TAD. DINOv2 and VideoMAEv2 have been pre-trained on large-scale image and video datasets, respectively.

In \figref{fig:dino_sen} and \figref{fig:vmae_sen}, we can see the only DINOv2 backbone can still achieve promising results on the HACS dataset (average mAP 33.6\%). However, the temporal-level context of VideoMAEv2 outperforms the spatial-level context of DINOv2 for action segments with different lengths. This aligns with our expectations, as the single spatial-level context lacks motion information and struggles to fully represent the action.

Additionally, in \figref{fig:dino_fp} and \figref{fig:vmae_fp}, we analyze the false positives for the two models. We can observe that the most significant types of false positive segments for the two models are \emph{localization error}. However, the impact of \emph{localization error} is higher for VideoMAEv2 compared to DINOv2, indicating that VideoMAEv2 has a greater potential to benefit from more accurate localization, despite having better overall performance (43.1\% vs 33.7\%). The results show that VideoMAEv2 alone still has room for improvement in localization, and there is a need to use spatial-level information to help improve it further.

\subsubsection{Comparison for VideoMAEv2 and the fused model of DINOv2 and VideoMAEv2}
In this section, we will analyze the effectiveness of the fused model (TriDet-Fused) by comparing it with the VideoMAEv2-only model.

In \figref{fig:vmae_sen} and \figref{fig:vmae_dino_sen}, we can observe that the average mAP for videos with more than one segment (\ie~the S, M, L items in the brown bar) increases. The accuracy of predicting short-length segments (\ie~XS, S in the pink bar) has also improved, though there is slightly less accuracy noticed for extremely long segments (\ie~XL in the pink bar).

Moreover, in \figref{fig:vmae_fn} and \figref{fig:vmae_dino_fn}, we can observe that the segments missed by TriDet using only VideoMAEv2 are mainly small ones (\ie~XS, S in the pink bar). However, this problem can be mitigated by combining VideoMAEv2 with the DINOv2, as demonstrated by the decrease in the false negative rates(\ie~\\35.9\% vs 35.7\% and 25.8\% vs 25.6\%, for XS and S respectively).

\begin{figure*}[t]
    \centering
    {
        \includegraphics[width=0.9\linewidth]{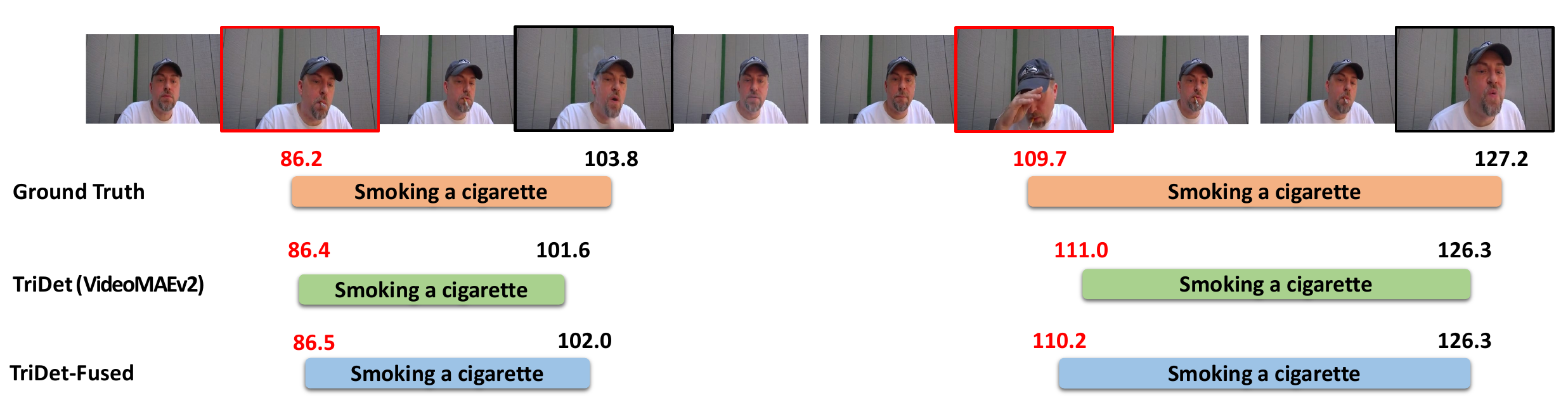}
        \includegraphics[width=0.9\linewidth]{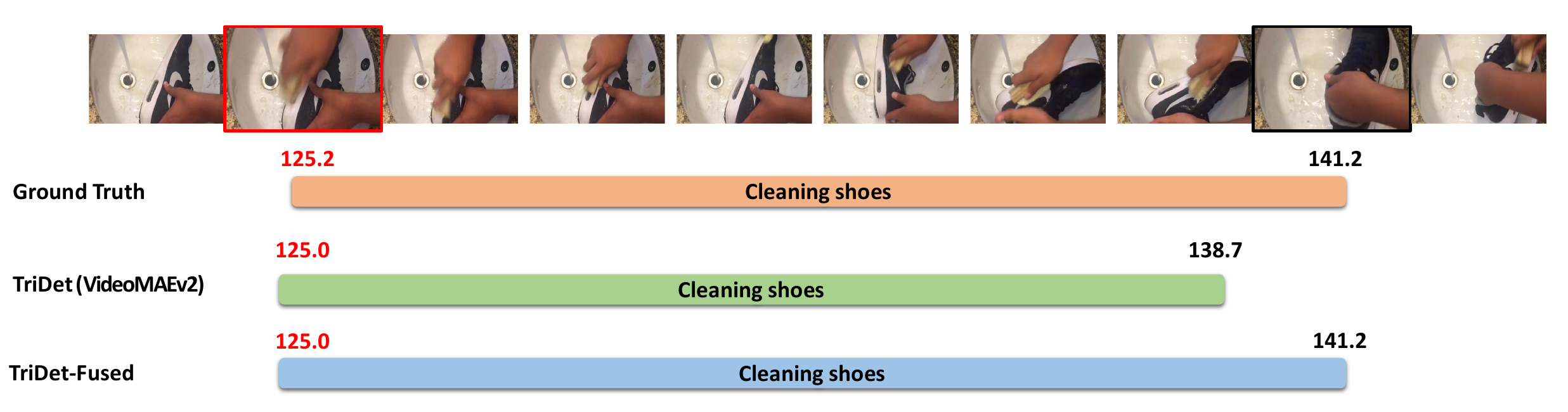}
        \includegraphics[width=0.9\linewidth]{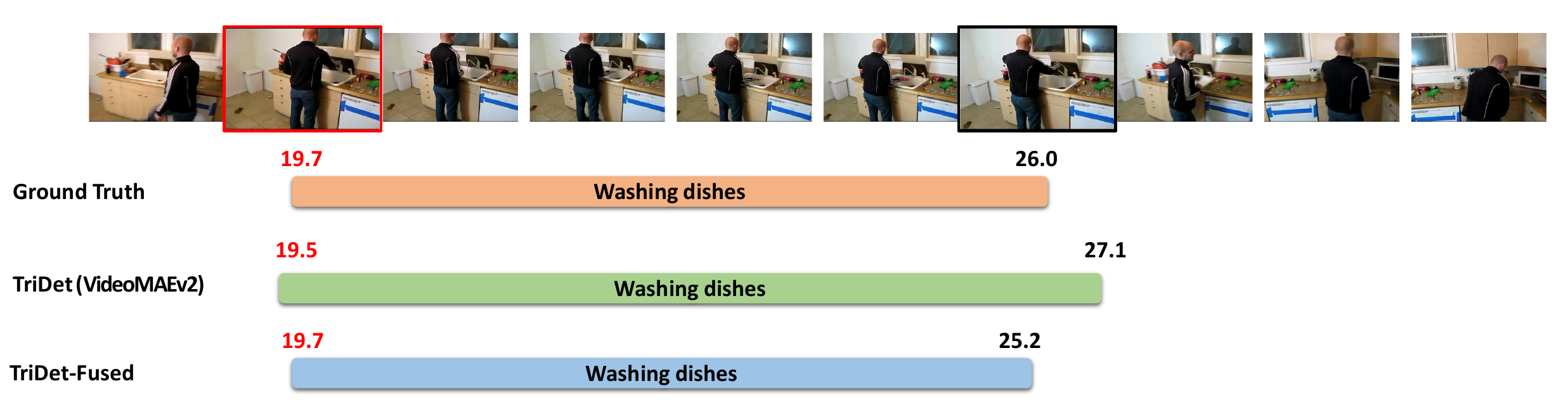}
}

  \caption{Visualization results of the HACS validation set. The start and end timestamps (in seconds) of the actions are highlighted in red and black text, respectively.}
  \label{fig:vis}
\end{figure*}

The results indicate that the DINOv2 improves\\TriDet's ability to detect multiple short segments from videos. However, for extremely long videos, the detector struggles to incorporate too much spatial-level context and still relies on the temporal-level backbone. This is because DINOv2 is trained on the image datasets, which makes it difficult to capture the motion dynamics in long videos. The results suggest that there is still room for improvement in the method.

In addition, in \figref{fig:vmae_fp} and \figref{fig:vmae_dino_fp}, after removing the imprecise localization prediction (\ie~localization error), the average mAP improves by $9.6\%$, indicating that even though TriDet achieved state-of-the-art performance, inaccurate localization is still the main problem. Moreover, background error is also an important issue. 
Thus, how to suppress meaningless predictions remains an open question. The improvement in removing error impact (\eg~backgroud error, wrong label error) also suggests that the introduction of DINOv2 pushes the upper limits of TriDet's capabilities.

\subsection{Qualitative Analysis}
In \figref{fig:vis}, we present the visualization of the detection results on the HACS validation set.Clearly, TriDet can accurately predict the start and end instants of the action segments. Moreover, for action segments that involve specific objects such as cigarettes and shoe brushes, TriDet-Fused shows better performance, indicating that decoupled fusion can enhance the accuracy of boundary prediction even further.

\section{Conclusion}
In this paper, we present TriDet, a one-stage convolutional-based framework. To enhance localization learning, we propose a Trident-head to model the action boundary using an estimated relative probability distribution around the boundary. We also address the rank-loss problem commonly found in transformer-based methods by introducing an efficient SGP layer. Additionally, we leverage pretrained large models to improve the discriminability in the video backbone and present a decoupled FPN that can further boost the detection accuracy. We evaluate our method on THUMOS14, HACS, Multi-THUMOS, and Charades datasets, showcasing its high generalization capability and state-of-the-art performance. We conduct extensive ablation studies to validate the effectiveness of our approach.

\begin{acknowledgements}
This work was supported by the National Key Research and Development Program of China under 2021YFB2700304, the National Natural Science Foundation of China under Grant No.61971017, No.U21A20523 and No. L222152, the National Key Research and Development Program of China under 2021ZD0140407 and the Special Project for Innovation in Next-Generation Electronic Information Technology under Grant No.20310105D.
\end{acknowledgements}

\appendices
\section{The rank-loss problem in Transformer.}
In~\citep{dong2021attention}, the authors discuss how the pure self-attention operation causes the input feature to converge to a rank-1 matrix at a double exponential rate, while MLP and residual connections can only partially slow this convergence. 
We have observed this phenomenon not only during initialization but also during training, which is disastrous for TAD tasks. This is because the video feature sequences extracted by pre-trained action recognition networks are often highly similar (see \figref{fig:cosine}), which further aggravates the rank-loss problem and makes the features at each instant indistinguishable, resulting in inaccurate detection of action.

We posit that the core reason for this issue lies in the softmax function used in self-attention. Namely, the probability matrix (\ie~softmax($QK^T$)) is \emph{nonnegative} and \emph{the sum of each row is 1}, indicating the outputs of SA are \emph{convex combination} for the value feature $V$. We demonstrate that the largest angle between any two features in $V' = SA(V)$ is always less than or equal to the largest angle between features in $V$.

\begin{definition}[Convex Combination]
Given a set of points $S=\{x_1, x_2..., x_n\}$, a convex combination is a point of the form $\sum_{n}{a_nx_n}$, where $a_n\geq0$ and $\sum_{n}{a_n}=1$.
\end{definition}

\begin{definition}[Convex Hull]
The convex hull $H$ of a given set of points $S$ is identical to the set of all their convex combinations. A convex hull is a convex set.
\end{definition}

\begin{property}[Extreme point]
An extreme point $p$ is a point in the set that does not lie on any open line segment between any other two points of the same set. For a point set $S$ and its convex hull $H$, we have $p\in S$.
\end{property}

\begin{lemma} 
\label{lemma:two}
Consider the case of a convex hull that does not contain the origin.
Let $a, b \in \mathbb{R}^n$ and let $S$ be the convex hull formed by them. Then, the angle between any two position vectors of points in $S$ is less than or equal to the angle between the position vectors of the extreme points $\vec{a}$ and $\vec{b}$.
\end{lemma}

\begin{proof}
Consider the objective function
\begin{equation}
    f(x) = \cos{(\vec{x},\vec{y})} = \frac{\langle\vec{x},\vec{y}\rangle}{\left\lVert\vec{x}\right\rVert_2\left\lVert\vec{y}\right\rVert_2},
\end{equation}
where $\vec{x},\vec{y}$ are the position vectors of two points $x_1, x_2$ within the convex hull $S$ (a line segment with extreme points $a$ and $b$). The angle between two vectors is invariant with respect to the magnitude of the vectors; thus, for simplicity, we define $\vec{x}=\vec{a}+x\vec{b}$, $\vec{y}=\vec{a}+y\vec{b}$, where $x,y \in[0,+\infty)$.
Moreover, we have 
\begin{equation}
\begin{aligned}
    f'(x) =&\left\lVert\vec{x}\right\rVert^{-3}_2 \left\lVert\vec{y}\right\rVert^{-1}_2 \times\\
    &[{\langle\Vec{b},\vec{y}\rangle||\vec{a}+x\vec{b}||_2^2-(||\vec{b}||^{2}_2x+\langle\vec{a},\vec{b}\rangle)\langle\vec{a}+x\vec{b},\vec{y}\rangle}]    
\end{aligned}
\end{equation}
We consider 
\begin{equation}
\begin{aligned}
    g(x)=&\langle\Vec{b},\vec{y}\rangle||\vec{a}+x\vec{b}||_2^2-(||\vec{b}||^{2}_2x+\langle\vec{a},\vec{b}\rangle)\langle\vec{a}+x\vec{b},\vec{y}\rangle \\
    =&\langle\vec{b},\vec{y}\rangle(||\vec{a}||_2^2+2\langle\vec{a},\vec{b}\rangle x+||\vec{b}||_2^2x^2)-[\langle\vec{b},\vec{y}\rangle||b||_2^2x^2\\
    &+(\langle\vec{a},\vec{b}\rangle||b||_2^2+\langle\vec{a},\vec{b}\rangle\langle\vec{b},\vec{y}\rangle)x+\langle\vec{a},\vec{y}\rangle\langle\vec{a},\vec{b}\rangle]\\
    =&(\langle\vec{a},\vec{b}\rangle\langle\vec{b},\vec{y}\rangle-\langle\vec{a},\vec{y}\rangle\langle\vec{b},\vec{b}\rangle)x + \langle\vec{a},\vec{a}\rangle\langle\vec{b},\vec{y}\rangle-\langle\vec{a},\vec{y}\rangle\langle\vec{a},\vec{b}\rangle.
\end{aligned}    
\end{equation}
Substituting $\vec{y}=\vec{a}+y\vec{b}$ into the above equation, we have 
\begin{equation}
\begin{aligned}
    g(x)=&(\langle\vec{a},\vec{b}\rangle\langle\vec{b},\vec{a}+y\vec{b}\rangle-\langle\vec{a},\vec{a}+y\vec{b}\rangle\langle\vec{b},\vec{b}\rangle)x + \\
    &\langle\vec{a},\vec{a}\rangle\langle\vec{b},\vec{a}+y\vec{b}\rangle-\langle\vec{a},\vec{a}+y\vec{b}\rangle\langle\vec{a},\vec{b}\rangle\\
    =&[\langle\vec{a},\vec{b}\rangle(\langle\vec{a},\vec{b}\rangle+y\langle\vec{b},\vec{b}\rangle)-(\langle\vec{a},\vec{a}\rangle+y\langle\vec{a},\vec{b}\rangle)\langle\vec{b},\vec{b}\rangle]x +\\
    & [\langle\vec{a},\vec{a}\rangle(\langle\vec{a},\vec{b}\rangle+y\langle\vec{b},\vec{b}\rangle)-(\langle\vec{a},\vec{a}\rangle+y\langle\vec{a},\vec{b}\rangle)\langle\vec{a},\vec{b}\rangle]\\
    =&(||\langle\vec{a},\vec{b}\rangle||_2^2-||\vec{a}||_2^2||\vec{b}||_2^2)x+(||\vec{a}||_2^2||\vec{b}||_2^2-||\langle\vec{a},\vec{b}\rangle||_2^2)y\\
    =&(||\langle\vec{a},\vec{b}\rangle||_2^2-||\vec{a}||_2^2||\vec{b}||_2^2)(x-y).
\end{aligned}
\end{equation}
According to the Cauchy-Schwartz inequality, we can obtain 
\begin{equation}
    ||\langle\vec{a},\vec{b}\rangle||_2^2-||\vec{a}||_2^2||\vec{b}||_2^2\leq0
\end{equation}
Then, we have
\begin{equation}
g(x)\left\{
\begin{aligned}
>0 & &x<y\\
=0 & &x=y \\
<0 & &x>y.
\end{aligned}
\right.    
\end{equation}
Thus, for any position vector $\vec{y}$, when $x=0$ or $x\rightarrow \infty$ (Equivalent to $\Vec{x} =\vec{a}$ or $\Vec{x} =\vec{b}$), the angle formed between $\Vec{y}$ and $\vec{x}$ is maximum.

Without loss of generality, given a specific $\vec{y}$, if its maximum vector $\vec{x}=\vec{a}$, we can then set $\vec{y}$ to $\vec{a}$ and find its maximum vector again, which yields
\begin{equation*}    \theta(\vec{x},\vec{y})\leq\theta(\vec{a},\vec{y})\leq\theta(\vec{b},\vec{a})
\end{equation*}
The proof is completed.
\end{proof}

\begin{theorem}
\label{theo:gen}
Consider the case of a convex hull that does not contain the origin. Let $X = \{x_1, x_2, \dots, x_k\}$ be a set of points, and let $S$ be its convex hull. Then, the maximum angle between the position vectors of any two points in $S$ is formed by the position vectors of two extreme points of $S$.
\end{theorem}

\begin{proof}
Assume that this case holds when k.

When $k=2$, based on Lemma \ref{lemma:two}, the maximum angle is formed by the extreme points $\vec{x_1}$ and $\vec{x_2}$.

When $k\geq3$, we can sort the elements of X such that for a point $y$ in $S$, $\vec{x_k}$ maximizes the angle $\theta(\vec{y},\vec{x_k})$. Furthermore, the points $x$ in $S$ are of the form:
\begin{equation}
\begin{aligned}
    &\lambda_1\vec{x_1}+\lambda_2\vec{x_2}+...+\lambda_k\vec{x_k}\\
    =&(\lambda_1+...+\lambda_{k-1})(\frac{\lambda_1 \vec{x_1}}{\lambda_1+...+\lambda_{k-1}}+...+\frac{\lambda_{k-1}\vec{x_{k-1}}}{\lambda_1+...+\lambda_{k-1}})\\
    &+\lambda_k\vec{x_k},
\end{aligned}
\end{equation}
where $(\frac{\lambda_1 \vec{x_1}}{\lambda_1+...+\lambda_{n-1}}+...+\frac{\lambda_{k-1}\vec{x_{k-1}}}{\lambda_1+...+\lambda_{k-1}})$ is a position vector of a point located within the convex hull induced by $\{x_1,x_2,...,x_{k-1}\}$. Through Lemma~\ref{lemma:two} and by definition, we can obtain
\begin{equation}
    \theta(\vec{x},\vec{y})\leq\theta(\vec{x_k},\vec{y})
\end{equation}
For any two points x and y in a convex hull S, by setting $\vec{y}=\Vec{x_k}$ and using the above inequality twice, without loss of generality, we can assume that the vector $\vec{x_1}$ makes the largest angle with $\vec{x_k}$. Then, we can obtain
\begin{equation}
\theta(\vec{x},\vec{y})\leq\theta(\vec{x_k},\vec{y})\leq\theta(\vec{x_1},\vec{x_k})
\end{equation}

By definition, $\theta(\vec{x_1},\vec{x_k})$ is no greater than the maximum angle formed by any other two basis vectors.

The proof is completed.
\end{proof}

\begin{corollary}
When the convex hull of the input set $V$ does not contain the origin, the largest angle between any two features after self-attention $V' = SA(V)$ is always less than or equal to the largest angle between features in $V$.
\end{corollary}

\begin{table*}[t]
\centering{
\caption{Comparison with the state-of-the-art methods on EPIC-KITCHEN dataset with the SlowFast backbone. \emph{V.} and \emph{N.} denote the \emph{verb} and \emph{noun} sub-tasks, respectively.}
\label{table:epic}
\setlength{\tabcolsep}{0.9mm}
\renewcommand{\arraystretch}{1.1}
\scalebox{0.9}{
\begin{tabular}{c|c|c c c c c c}
\toprule
 Subset &Method & 0.1 & 0.2 & 0.3 & 0.4 & 0.5 &Avg. \\
\midrule
\multirow{4}*{\tabincell{c}{\emph{V.}}}
&BMN~\citep{lin2019bmn}& 10.8 & 8.8 & 8.4 & 7.1& 5.6 & 8.4 \\
&G-TAD~\citep{xu2020g}& 12.1 & 11.0 & 9.4 & 8.1 & 6.5 & 9.4 \\
&ActionFormer~\citep{zhang2022actionformer}& 26.6 & 25.4 & 24.2 & 22.3 & 19.1 & 23.5 \\
&\textbf{TriDet}&\textbf{ 28.6}  & \textbf{27.4}  & \textbf{26.1} & \textbf{24.2}  & \textbf{20.8}  & \textbf{25.4} \\
\midrule
\multirow{4}*{\tabincell{c}{\emph{N.}}}
&BMN ~\citep{lin2019bmn}& 10.3 & 8.3 & 6.2 & 4.5 & 3.4 & 6.5 \\
&G-TAD~\citep{xu2020g}& 11.0 & 10.0 & 8.6 & 7.0 & 5.4 & 8.4 \\
&ActionFormer~\citep{zhang2022actionformer}& 25.2 & 24.1 & 22.7 & 20.5 & 17.0 & 21.9 \\
&\textbf{TriDet}& \textbf{27.4}  & \textbf{26.3}  & \textbf{24.6} & \textbf{22.2}  & \textbf{18.3}  & \textbf{23.8} \\
\bottomrule
\end{tabular}
}}
\end{table*}

\begin{remark}
In the temporal action detection (TAD) task, the temporal feature sequences extracted by the pretrained video classification backbone often exhibit high similarity and pure layer normalization~\citep{ba2016layer} projects the input features onto the hypersphere in the high-dimensional space. Consequently, the convex hull induced by these features often does not encompass the origin. As a result, the self-attention operation causes the input features to become more similar, reducing the distinction between temporal features and hindering the performance of the TAD task.
\end{remark}


\begin{table*}[t]
\centering{
\caption{Comparison with the state-of-the-art methods on ActivityNet-1.3 dataset.}
\label{table:activitynet}
\setlength{\tabcolsep}{1.3mm}
\renewcommand{\arraystretch}{1.1}
\scalebox{0.9}{
\begin{tabular}{c| c | c c c c c }
\toprule
Method & Backbone& 0.5 & 0.75 & 0.95 & Avg. \\
\midrule
PGCN~\citep{zeng2019graph} & I3D & 48.3 & 33.2 & 3.3 & 31.1\\
ReAct~\citep{shi2022react} & TSN & 49.6 & 33.0 & 8.6 & 32.6\\
BMN~\citep{lin2019bmn}& TSN & 50.1 & 34.8 & 8.3 & 33.9 \\
G-TAD~\citep{xu2020g}& TSN & 50.4 & 34.6 & 9.0 & 34.1\\
AFSD~\citep{lin2021learning} & I3D & 52.4 & 35.2 & 6.5 & 34.3\\
TadTR~\citep{liu2022end} & TSN & 51.3  & 35.0  & 9.5  & 34.6\\
TadTR~\citep{liu2022end} & R(2+1)D & 53.6  & 37.5  & 10.5  & 36.8\\
VSGN~\citep{zhao2021video} & I3D & 52.3 & 35.2 & 8.3 & 34.7\\
PBRNet~\citep{liu2020progressive} & I3D & 54.0 & 35.0 & 9.0 & 35.0\\
TCANet+BMN~\citep{qing2021temporal} & TSN & 52.3 & 36.7 & 6.9 & 35.5\\
TCANet+BMN~\citep{qing2021temporal} & SlowFast & 54.3 & \textbf{39.1} & \textbf{8.4} & \textbf{37.6}\\
TALLFormer~\citep{cheng2022tallformer} & Swin & 54.1 & 36.2 & 7.9 & 35.6\\
ActionFormer~\citep{zhang2022actionformer} & R(2+1)D & \textbf{54.7} & 37.8 & \textbf{8.4} & 36.6\\
\textbf{TriDet} & R(2+1)D &\textbf{54.7} & \textbf{38.0} & \textbf{8.4} & \textbf{36.8}\\
\bottomrule
\end{tabular}
}}
\end{table*}
\begin{table*}[t]
\centering{
\caption{Comparison with the state-of-the-art methods on the Charades datasets.}
\label{table:chrades}
\setlength{\tabcolsep}{1.6mm}
\renewcommand{\arraystretch}{1.1}
\begin{tabular}{c|c|cccc}
\toprule
\multirow{2}{*}{Method} & \multirow{2}{*}{Backbone}  & \multicolumn{4}{c}{Charades} \\
\cline{3-6}
                        &                               & 0.2    & 0.5    & 0.7   & Avg    \\
\midrule
PointTAD~\citep{tan2022pointtad}                & I3D    &  17.5  & 13.5  & 9.1  & 12.1  \\
ASL~\citep{shao2023action}& I3D                  & 24.5  & 16.5  & 9.4  & 15.4 \\
\midrule
\textbf{TriDet}                  & I3D          & \textbf{27.1}  & \textbf{20.4}  & \textbf{13.2} & \textbf{18.4}  \\
\bottomrule
\end{tabular}
}
\end{table*}


\section{Additional Experimental Results}
To further validate the robustness of TriDet, we conduct additional experiments on two single-label datasets: ActivityNet-1.3\citep{caba2015activitynet} and EPIC-KITCHEN 100~\citep{Damen2022RESCALING}, and a multilabel dataset Charades~\citep{sigurdsson2016hollywood}. 

EPIC-KITCHEN 100 is a large-scale dataset of first-person vision that has two subtasks: \emph{noun} localization (\eg~door) and \emph{verb} localization (\eg~open the door). It contains 495 and 138 videos with 67,217 and 9,668 action segments for training and test, respectively. The number of action classes for \emph{noun} and \emph{verb} are 300 and 97. ActivityNet shares 200 classes of action with the HACS dataset and contains 10,024 videos for training, as well as 4,926 videos for test. Charades is a large-scale common household activities dataset that contains 7,985 and 1,863 videos for training and test, with 49,809 and 16,691 action segments, respectively.

For EPIC-KITCHEN, we report IoU thresholds at [0.1:0.5:0.1]. For ActivityNet, we report the result at the IoU threshold [0.5, 0.75, 0.95], and the average mAP is computed at [0.5:0.95:0.05]. We report the average IoU with thresholds [0.1: 0.1: 0.9] for the Charades datasets. 

The initial learning rate is set to $10^{-4}$ for Charades and EPIC-KITCHEN, and  $10^{-3}$ for ActivityNet. We train for 9, 23, 19 and 15 for Charades, EPIC-KITCHEN \emph{verb}, EPIC-KITCHEN \emph{noun}, ActivityNet (including a warmup of 5, 5, 5 and 10 epochs).

For ActivityNet, the number of bins $B$ of the Trident-head is set to 12, the convolution window $w$ is set to 15 and the scale factor $k$ is set to 1.3. For Charades, and EPIC-KITCHEN, the number of bins $B$ of the Trident-head is set to 16,  the convolution window $w$ is set to 1, and the scale factor $k$ is set to 5 for Charades and 1.5 for EPIC-KITCHEN.

We show their results in \tabref{table:epic}, \tabref{table:activitynet} and \tabref{table:chrades}, respectively.
With the same backbone network, TriDet achieves State-of-the-art performance on these datasets, demonstrating its robustness.


\bibliographystyle{spbasic}      
\bibliography{bib}

%
%

\end{document}